\documentclass[twoside,11pt]{article}
\usepackage{journal2e_v1}

\usepackage{epsf}
\usepackage{epsfig}
\usepackage{amsmath}
\usepackage{subfigure}
\usepackage{amssymb}
\usepackage{algorithm}
\usepackage{multirow}
\usepackage{multicol}

\usepackage{graphicx} % more modern
\usepackage{algorithm}
\usepackage{algorithmic}

\def\a{{\bf a}}

\def\b{{\bf b}}
\def\C{{\bf C}}

\def\X{{\bf X}}

\def\s{{\bf s}}

\def\x{{\bf x}}
\def\y{{\bf y}}
\def\z{{\bf z}}

\def\u{{\bf u}}

\def\v{{\bf v}}

\def\0{{\bf 0}}
\def\1{{\bf 1}}

\def\AM{{\mathcal A}}

\def\XM{{\mathcal X}}

\def\RB{{\mathbb R}}
\def\EB{{\mathbb E}}
\def\NB{{\mathbb N}}

\def\ome{\mbox{\boldmath$\omega$\unboldmath}}

\def\epsi{\mbox{\boldmath$\epsilon$\unboldmath}}
\def\etab{\mbox{\boldmath$\eta$\unboldmath}}

\def\tha{\mbox{\boldmath$\theta$\unboldmath}}
\def\Tha{\mbox{\boldmath$\Theta$\unboldmath}}

\def\argmin{\mathop{\rm argmin}}
\def\liml{\mathop{\lim}\limits}
\def\liminfl{\mathop{\liminf}\limits}
\def\limsupl{\mathop{\limsup}\limits}
\def\minl{\mathop{\min}\limits}
\def\infl{\mathop{\inf}\limits}

\def\Ke{\mathrm{Ke}}
\def\sgn{\mathrm{sgn}}

\def\GIG{\mathrm{GIG}}
\def\Ga{\mathrm{Ga}}
\def\IG{\mathrm{IG}}
\def\EP{\mathrm{EP}}

\def\Lv{\mathrm{Lv}}

\title{Kinetic Energy Plus Penalty Functions for Sparse Estimation}

\author{
\name Zhihua Zhang and Shibo Zhao \\
\addr Department of Computer Science \& Engineering \\
Shanghai Jiao Tong University \\
800 Dong Chuan Road, Shanghai, China 200240 \\
\texttt{zhzhang@gmail.com} \AND
\name  Zebang Shen  \\
\addr College of Computer Science \& Technology \\
Zhejiang University \\
38 Zheda Road, Hangzhou, China 310028 \\
\texttt{shenzebang@gmail.com} \AND
\name Shuchang Zhou \\
\addr Key Laboratory of Computer System and Architecture \\
Institute of Computing Technology \\
Chinese Academy of Sciences, Beijing, China  \\
\texttt{shuchang.zhou@gmail.com}
}

\date{\today}

\begin{document}
\maketitle

\begin{abstract}%
Motivated by iteratively reweighted $\ell_q$  methods,
we propose and study a family of sparsity-inducing penalty functions.
Since the penalty functions are related to the kinetic energy in special relativity, we  call them \emph{kinetic energy plus} (KEP) functions.
We construct the KEP function by using the concave conjugate of a $\chi^2$-distance function
and present several novel insights  into the  KEP function with $q=1$.
In particular, we derive a thresholding operator based on the  KEP function, and prove its mathematical properties and asymptotic properties in sparsity modeling. Moreover, we show that a coordinate descent algorithm is especially appropriate for the  KEP function.
Additionally, we discuss the relationship of KEP with the  penalty functions $\ell_{1/2}$ and MCP.
The theoretical and empirical analysis validates that the KEP function is effective and efficient in high-dimensional data modeling.
\end{abstract}

\begin{keywords} iteratively reweighted minimization methods, kinetic energy plus penalties, nonconvex penalization, stability,
concave conjugate
\end{keywords}

\section{Introduction}

Sparsity is an important attribute in statistical modeling for high-dimensional
data sets, especially when the underlying model has a sparse representation.
Typically, the penalty theory has been  used for capturing sparsity.
A principled approach is to employ the $\ell_1$-norm penalty as a convex relaxation of the $\ell_0$-norm~\citep{TibshiraniLASSO:1996}.
Additionally, some nonconvex alternatives, such as the bridge  penalty $\ell_q$ ($q\in (0,1)$), the log-penalty~\citep{MazumderSparsenet:11}, the nonconvex EXP~\citep{BradleyICML:1998,GaoAAAI:2011},
the smoothly clipped absolute deviation (SCAD)
penalty~\citep{Fan01} and the minimax concave plus (MCP) penalty~\citep{Zhang2010mcp},
have attracted wide attention.

On one hand,  nonconvex penalties usually have nice consistency properties~\citep{Fan01,ZhangZhang2012}.
On the other hand, they would yield computational challenges due to their nonconvexity and nondifferentiability.
In order to address this challenge,  \cite{Fan01} proposed a local quadratic approximation (LQA), while \cite{ZouLi:2008}
then devised a local linear approximation (LLA).
These methods enjoy
a so-called majorization-minimization (MM) procedure~\citep{Lange:2000,HunterLiAS:2004}.
In the same spirit,  iteratively reweighted $\ell_q$ ($q=2$ or $1$) methods have  been also developed to
find sparse solutions~\citep{ChartrandICASSP:2008,CandesWakinBoyd:2008,WipfNIPS:2008,Daubechies:2010,WipfNagarajan:2010}.
Additionally, \citet{MazumderSparsenet:11} developed a SparseNet algorithm  based on  coordinate descent for the MCP penalty.

Our work is mainly motivated by the iteratively reweighted $\ell_q$ method of \cite{Daubechies:2010} and by the coordinate descent
algorithm of \citet{MazumderSparsenet:11}.
\cite{Daubechies:2010} demonstrated the elegant performance of their  method theoretically and empirically.
However, there are still several  issues that deserve to  be further studied. First, the penalty function corresponding
to the method is not explicitly available. This results in that the corresponding thresholding operator is also unknown.
Second, it is  unclear whether the  estimator has some properties such as unbiasedness, continuity and asymptotic consistency.

Within and beyond these issues, we develop a family of novel penalty functions.
First,  we derive the expression of  the  penalty function
by using the concave conjugate of a $\chi^2$-distance function. Interestingly, when $q=2$, the
expression is mathematically the same with the kinetic energy in special relativity. We thus refer to them as  \emph{kinetic energy plus} (KEP) functions. We explore the connection of the KEP penalty
with the $\ell_q$-norm and $\ell_{q/2}$-norm.
The constructive method
encourages us to
rederive the  iteratively reweighted $\ell_q$ method of \cite{Daubechies:2010} via an augmented Lagrangian methodology.

In this paper we are especially concerned with the case of $q=1$,  because the corresponding
KEP penalty is nonconvex. Theoretically, we give mathematical properties and  asymptotic behaviors of the resulting estimators built on the work of \citet{Fan01,KnightFu:2000,ZhaoYu:2006,ZouLi:2008}. Specifically, the asymptotic behaviors are studied  both in the conventional fixed $p$ (the number of features) setting and in the large $p$
setting as $n$ (the training sample size) increases.

Computationally, we develop the corresponding thresholding operator. We show that the
thresholding operator bridges the soft thresholding operator based on the lasso and the half thresholding operator based on
the $\ell_{1/2}$ penalty~\citep{XuTNN:2012}. However, compared with the soft thresholding operator, our thresholding operator
has  unbiasedness and oracle properties. Compared with the half thresholding operator, our thresholding operator
is  continuous, which  makes it stable in model prediction.
These properties assure that the  KEP function is  suitable for coordinate descent algorithms.
Moreover, the convergence property of the coordinate descent algorithm can be ensured~\citep{MazumderSparsenet:11}.

We uncover an inherent connection between the KEP and MCP functions.
Specifically, the MCP function can be also defined as the concave conjugate of the $\chi^2$-distance function.
The difference between  KEP and MCP  is then due to  asymmetricity of the $\chi^2$-distance function.
This difference makes the KEP outperform MCP in that KEP enjoys a nesting property---a desirable property
stated by~\citet{MazumderSparsenet:11}.

It is worth noting that \cite{PalmerNIPS:2006}  and \citet{WipfNIPS:2008}  considered the application of concave conjugates  for non-Gaussian latent variable models. The notion of concave conjugates has been also used
by \citep{TZhangJMLR:10,ZhangNIPS:2012,ZhangAAAI:2013} in construction of nonconvex penalty functions.
For example, \citet{ZhangAAAI:2013} employed the concave conjugate of  the squared Euclidean distance function  for defining  the MCP function.
\citet{ZhangNIPS:2012} then showed that the nonconvex LOG and EXP functions can be defined as the concave conjugate of the Kullback-Leibler (KL) divergence. Interestingly, asymmetricity of the KL divergence implies the  connection between LOG and EXP, which stands in parallel with the connection between KEP and MCP.

The remainder of the paper is organized as follows.
Section~\ref{sec:problem}
reviews the iteratively reweighted $\ell_q$ method of \cite{Daubechies:2010}.
We propose  the KEP penalty in Section~\ref{sec:method0},
and study sparse estimation based on the KEP function in Section~\ref{sec:sest}. In Section~\ref{sec:related} we explore the relationship between MCP and KEP. In Section~\ref{sec:asymp} we give asymptotic consistent results of sparse estimators.
In Section~\ref{sec:experiment} we conduct our
experimental evaluations. Finally, we conclude our work in
Section~\ref{sec:conclusion}. Some proofs are given in Appendix.

\section{Problem Formulations} \label{sec:problem}

Typically, supervised learning can be formulated as an optimization problem under the regularization framework
or penalty theory:
%\begin{equation} \label{eqn:1}
\[
\min_{\tha} \; \Big\{L(\tha; \XM) +   P(\tha; \lambda)\Big\},
\]%\end{equation}
where ${\XM}=\{(\x_i, y_i); i=1, \ldots, n\}$ is a training dataset,
$\tha$  the model parameter vector, $L(\cdot)$ the loss function penalizing data misfit, $P(\cdot)$
the regularization term penalizing model complexity, and $\lambda$ ($>0$) the tuning parameter of balancing the relative significance
of the loss function and the penalty.

%%%%%%%%%%%%%%%%%%%%%%%%%%%%%%%%%%%%%%%%%%%%%%%%%%%%%%%%%%%%%%%%%%%%%%%%%%%%%%%%%%%%%%%
%%%%%%%%%%%%%%%%%%%%%%%%%%%%% TO BE CHANGED? %%%%%%%%%%%%%%%%%%%%%%%%%%%%%%%%%%%%%%%%%%
%%%%%%%%%%%%%%%%%%%%%%%%%%%%%%%%%%%%%%%%%%%%%%%%%%%%%%%%%%%%%%%%%%%%%%%%%%%%%%%%%%%%%%%

The choice of the loss function depends very much on the supervised learning problem at hand.
Our presentation is mainly based on the linear regression problem
\[
L(\b; \XM) = \frac{1}{2}\sum_{i=1}^n (y_i{-}\x_i^T\b)^2 =\frac{1}{2} \|\y {-} \X \b\|_2^2,
\]
where $\y=(y_1, \ldots, y_n)^T \in \RB^n$, $\X=[\x_1, \ldots, \x_n]^T \in \RB^{n{\times}p}$, and $\b=(b_1, \ldots, b_p)^T \in \RB^p$.
We can also consider extensions
involving other exponential family models.

%%%%%%%%%%%%%%%%%%%%%%%%%%%%%%%%%%%%%%%%%%%%%%%%%%%%%%%%%%%%%%%%%%%%%%%%%%%%%%%%%%%%%%%

A widely used setting for  penalty is $P(\b;\lambda)=\lambda \sum_{j=1}^p P_{j}(b_j)$,
which implies that the penalty function consists of $p$ separable subpenalties and all subpenalties share a common tuning parameter $\eta$.
In order to find a sparse solution of $\b$, one imposes the $\ell_0$-norm penalty to  $\b$.
However, the resulting optimization problem is usually NP-hard. Thus,
the $\ell_1$-norm penalty  $P(\b; \lambda)=\lambda \|\b\|_1= \lambda \sum_{j=1}^p |b_j|$ is an effective convex  alternative. Additionally,
some nonconvex alternatives, such as the bridge penalty $\ell_q$ ($q\in (0,1)$),
SCAD and MCP, have been  employed. Meanwhile, iteratively reweighted $\ell_q$  ($q=1$ or $2$) minimization methods were developed for finding
sparse solutions.

Specifically, \cite{Daubechies:2010} proposed an iteratively reweighted least-squares (or $\ell_2$) minimization method. This method can be slightly modified as an  iteratively reweighted $\ell_1$ minimization version. Thus, we here consider a general $\ell_q$  version.
In particular,
the method introduces a set of auxiliary variables, including  a real number $\epsilon>0$ and a weight vector $\ome=(\omega_1, \ldots, \omega_p)^T \in \RB^p$ with $\omega_j>0$ for $j=1, \ldots, p$. Subsequently,  the  iteratively reweighted $\ell_q$ method solves the following optimization problem~\footnote{\cite{Daubechies:2010} originally considered the  iteratively reweighted $\ell_q$ method for  a sparse recovery problem with
equality constraints. The method  also applies to the problem in the presence of noise, that is, the problem in (\ref{eqn:prob01}).}:
\begin{equation}
\label{eqn:prob01}
\min \; \bigg \{ J(\b, \ome, \epsilon)
:= \frac{1}{2} \|\y {-} \X \b \|_2^2  +  \frac{\lambda}{2} \sum_{j=1}^p \Big[|b_j|^q \omega_j + (\epsilon^2 \omega_j {+} \omega_j^{-1}) \Big] \bigg\},
\end{equation}
where $\lambda >0$.
%Here $q=1$ or $2$, which   corresponds to iteratively reweighted $\ell_1$ or $\ell_2$ minimization.
Furthermore, given the $t$th estimates $(\b^{(t)}, \ome^{(t)}, \epsilon^{(t)})$, one
 recursively defines
\begin{align}
\label{eqn:ms1}
\b^{(t{+}1)} & = \argmin_{\b}   J(\b, \ome^{(t)}, \epsilon^{(t)}) \nonumber \\
& = \argmin_{\b}  \frac{1}{2} \|\y {-} \X \b\|_2^2+  \frac{\lambda}{2}  \sum_{j=1}^p |b_j|^q \omega^{(t)}_j
\end{align}
and
\begin{align}
\label{eqn:as1}
\ome^{(t+1)} & = \argmin_{\ome>0}   J(\b^{(t{+}1)}, \ome, \epsilon^{(t+1)}) \nonumber \\
&= \argmin_{\ome>0}  \;   \sum_{j=1}^p \Big[\big|b_j^{(t+1)}\big|^q \omega_j {+}   \frac{ (\epsilon^{(t{+}1)})^2 \omega_j^2 + 1} {\omega_j} \Big],
\end{align}
where $\epsilon^{(t+1)} = \min \big(\epsilon^{(t)}, r_{k{+}1}(\b^{(t)})/p \big)$.
Here $k$ ($<p$) is a prespecified positive integer and $r_i(\b)$ is the $i$th largest element of the vector $(|b_1|, \ldots, |b_p|)^T$.
It is directly obtained that
\[
\omega_{j}^{(t{+}1)} = \frac{1} {\sqrt{ |b_j^{(t+1)}|^q +  (\epsilon^{(t{+}1)})^2 }}, \quad j=1, \ldots, p.
\]

\cite{Daubechies:2010} demonstrated the performance of the iteratively reweighted $\ell_q$ method theoretically and empirically. However, there are still several questions that would be interesting. For example,
\begin{enumerate}
\item[(1)]  What are the explicit expressions of the penalty function and its corresponding thresholding operator for the penalized regression problem in (\ref{eqn:prob01})?
\item[(2)] Does  the estimator resulted from the problem in (\ref{eqn:prob01}) have  properties such as ``unbiasedness," ``sparsity," and  ``continuity," and ``asymptotic consistency?" %which have been proposed by \cite{Fan01}?
%\item[(3)] Are there  Bayesian interpretations in the penalized regression problem  (\ref{eqn:prob01})?
\end{enumerate}
In this paper we introduce penalty functions that we call \emph{kinetic energy plus} (KEP) functions to address these questions.
In Section~\ref{sec:method0} we derive the KEP function  by using the concave conjugate of a $\chi^2$-distance function.
In Section~\ref{sec:sest} we develop a sparse estimation approach based on the KEP penalty with $q=1$ and
present some important properties for assisting our approach.
In Section~\ref{sec:asymp} we present asymptotic consistent results about the sparse estimator.
Thus,  our work  not only deals with the  questions mentioned above but also
provides  new  insights into sparse estimation problems well.

\section{Kinetic Energy Plus (KEP) Penalty Functions}
%\section{Methodology}
\label{sec:method0}

Before presenting our work, we first give some notations. We denote
$\RB_{+}^p=\{\u =(u_1, \ldots, u_p)^T \in \RB^p: u_j \geq 0 \mbox{ for } j=1, \ldots, p\}$
and $\RB_{++}^p=\{\u =(u_1, \ldots, u_p)^T \in \RB^p: u_j > 0 \mbox{ for } j=1, \ldots, p\}$. Furthermore, if $\u \in \RB_{+}^p$ (or $\u \in \RB_{++}^p$),
we also write $\u\geq 0$ (or $\u > 0$). Additionally,  we denote $|\u|^q=(|u_1|^q, \ldots, |u_p|^q)^T$ and $\|\u\|_q^q = \sum_{j=1}^p |u_j|^q$.

We observe that the minimization problem in (\ref{eqn:as1})
is equivalent to the following problem
\[
\min_{\ome>0} \;   \sum_{j=1}^p \frac{1}{2} \Big[|b_j|^q \omega_j +  \frac{(\omega_j \epsilon-1)^2}{\omega_j} \Big].
\]
By direct calculations,  the corresponding minimum is given by
\begin{equation} \label{eqn:rp1}
\sum_{j=1}^p \big(\sqrt{|b_j|^q + \epsilon^2} -\epsilon \big).
\end{equation}

We are now  able to answer  the first question given in Section~\ref{sec:problem}. That is,
$\big(\sqrt{|b_j|^q + \epsilon^2} -\epsilon \big)$ is the penalty associated with the iteratively reweighted $\ell_q$ minimization method of \cite{Daubechies:2010}. In other words,  the method is used to solve  the following penalization problem:
\begin{equation} \label{eqn:doubl}
\min_{\b} \; \Big\{L(\b; \XM)  +   \lambda \sum_{j=1}^p \big(\sqrt{|b_j|^q + \epsilon^2} -\epsilon \big)\Big\},
\end{equation}
which can in turn be formulated into the optimization problem  in (\ref{eqn:prob01}).

We now present an alternative derivation of the above penalty function and establish its connection  with the kinetic energy in special relativity.
Note that $\frac{(\omega \epsilon-1)^2}{\omega}$ is related to the $\chi^2$-distance. For nonnegative $\omega$ and $\eta$,
the $\chi^2$-distance between them is  $\frac{(\omega-\eta)^2}{\omega}$.
This  motivates us to develop a new approach for the construction of KEP penalty functions.

We first study a nonseparable version. In this case,  given  $\alpha>0$ and $\eta>0$,
we consider the following optimization problem
%\begin{equation}
\[
\min_{\omega>0} \;   \omega \|\b\|_q^q + \frac{1}{2 \alpha} \frac{(\omega-\eta)^2}{\omega}.
\] %\end{equation}
It is immediate that the corresponding minimum is given by
\[
\frac{\eta}{\alpha}  \Big(\sqrt{2 \alpha \|\b\|_q^q {+}1} -1\Big) \quad (\mbox{denoted } \; {\Ke}(|\b|^q; \eta, \alpha)).
\]
Interestingly, if $q=2$, $p=3$, $2 \alpha=1/(m^2 c^2)$ and $2 \eta=1/m$ where $m>0$ is the  mass at rest  and $c>0$ is the velocity of light,
we can obtain that
\[
 {\Ke} =  \sqrt{m^2 c^4 {+} c^2 \|\b\|_2^2} - m c^2,
\]
which is the kinetic energy in relativity theory.

We next study a separable version.
Alternatively, we are concerned with the following optimization problem
%\begin{equation}
\[
\min_{\ome>0} \; Q(\ome|\b, \etab) :=  \ome^T |\b|^q + \frac{1}{
2 \alpha} \sum_{j=1}^p \frac{(\omega_j-\eta)^2}{\omega_j}.
\] %\end{equation}
Let $C(|\b|^q)$ denote the minimum  of the above problem, which is the concave conjugate of  $-\frac{1}{2 \alpha} \sum_{j=1}^p \frac{(\omega_j-\eta)^2}{\omega_j}$ with respect to (w.r.t.) $|\b|^q$.
It is easily computed that
%\begin{equation} \label{eqn:cf1}
\[
C(|\b|^q)  = \sum_{j=1}^p \frac{\eta}{\alpha}\Big(\sqrt{2 \alpha |b_j|^q {+}1} -1\Big)
\] %\end{equation}
at $\hat{w}_j= \eta (2 \alpha |b_j|^q {+}1)^{-1/2}$.
With  $C(|\b|^q)$ as the penalty,
the corresponding iteratively reweighted $\ell_q$ minimization method  is then used to solve  the following penalization problem:
\begin{equation} \label{eqn:p3}
\min_{\b} \; \Big\{J(\b) := L(\b; \XM)  + \sum_{j=1}^p \frac{\eta}{\alpha}\big(\sqrt{2 \alpha |b_j|^q {+}1} -1\big) \Big\},
\end{equation}
which can in turn be formulated as the optimization problem:
\begin{equation} \label{eqn:p2}
\min_{\b}  \min_{\ome > 0} \bigg\{L(\b; \XM) + \ome^{T} |\b|^q +  \frac{1}{2 \alpha}   \sum_{j=1}^p \frac{(\omega_j-\eta)^2}{\omega_j}  \bigg\}.
\end{equation}
Clearly, when we set $2 \eta=\lambda /\epsilon$ and $2 \alpha = 1/\epsilon^2$, the problems  (\ref{eqn:p3}) and (\ref{eqn:p2}) are respectively  equivalent to  (\ref{eqn:doubl}) and (\ref{eqn:prob01}).
In this case, we further see that $\epsilon =  m c$ and $\lambda = {c}$.
Thus,  we have  very interesting physical meanings of the hyperparameters $\lambda$ and $\epsilon$
in the iteratively reweighted least squares method of  \citet{Daubechies:2010}.

In this paper  we define the following penalty function:
\begin{equation} \label{eqn:relativity}
\Psi(|b|^q; \eta, \alpha) =  \frac{\eta}{\alpha}\Big(\sqrt{2 \alpha |b|_q^q {+}1} -1\Big).
\end{equation}
We  refer to it as the \emph{kinetic energy plus} (KEP) function of $b$, due to the relationship with the kinetic energy in relativity theory.
To explore the relationship of $\Psi(|b|^q; \eta, \alpha)$ with the $\ell_q$-norm,
we let $\eta = \frac{\lambda \alpha}{\sqrt{2 \alpha {+}1} {-}1}$ for some $\lambda>0$. Accordingly, we define
\begin{equation} \label{eqn:rho}
\Phi(|b|^q; \alpha) = \frac{\sqrt{2 \alpha |b|^q +1} -1} {\sqrt{2 \alpha +1} -1} =  \frac { (\sqrt{2\alpha {+} 1} {+}1) |b|^q }{\sqrt{2\alpha |b|^q +1} +1},
\end{equation}
which goes through the points $(0, 0)$ and $(1, 1)$ like the $\ell_q$-norm. The derivative of $\Phi(|b|^q; \alpha)$ w.r.t.\ $|b|^q$
is
\[
\Phi'(|b|^q;  \alpha) := \frac{\partial \Phi(|b|^q;  \alpha)}{\partial |b|^q}  =  \frac {\sqrt{2\alpha {+} 1} {+}1}{ 2 \sqrt{2\alpha |b|^q +1}}.
\]
We now present the  following proposition.

\begin{proposition} \label{lem:1}
Let  $\Phi(|b|^q;\alpha)$ be defined in   (\ref{eqn:rho}). Then,
\begin{enumerate}
\item[\emph{(i)}] $\Phi(|b|^q; \alpha)$
is a nonnegative, nondecreasing and concave function of $|b|^q$.
\item[\emph{(ii)}]  $ \lim_{\alpha \rightarrow \infty} \Phi(|b|^q;  \alpha)
= |b|^{\frac{q}{2}}$  and  $\lim_{\alpha \rightarrow \infty}  \; \Phi'(|b|^q;  \alpha)
= \frac{ 1}{2|b|^{\frac{q}{2}}}$.
\item[\emph{(iii)}]    $\lim_{\alpha \rightarrow 0+} {\Phi}(|b|^q;  \alpha)
= |b|^{q}$   and     $\lim_{\alpha \rightarrow 0+}  \;  \Phi'(|b|^q;  \alpha)
= 1$.
%\item[\emph{(iii)}]  $ \lim_{\alpha \rightarrow \infty} \frac{\partial \Phi(|b|^q;  \alpha)}{\partial |b|^q}
%= \frac{ 1}{2|b|^{q/2}}$   and   \\  $\lim_{\alpha \rightarrow 0+}  \frac{\partial \Phi(|b|^q;  \alpha)}{\partial |b|^q}
%= 1$
\end{enumerate}
\end{proposition}

The proof is immediately.
This proposition shows that $\Phi(|b|^q; \alpha)$ can be regarded as a penalty  for $b$.
Specifically, $\Phi(|b|^2; \alpha)$ (i.e., $q=2$)
defines a convex
penalty of $b$, while
$\Phi(|b|; \alpha)$ (i.e., $q=1$) defines a nonconvex penalty of $b$. Moreover, Proposition~\ref{lem:1} says that $\Phi(|b|^q; \alpha)$
bridges the $\ell_{q}$-norm and the $\ell_{q/2}$-norm.    Figure~\ref{fig:penalty} illustrates  $\Phi(|b|^q; \alpha)$
when $p=1$ and $p=2$.

\begin{figure}[!ht]
\centering
\subfigure[$\Phi(|b|; \alpha)$ vs. $|b|$]{\includegraphics[width=65mm,height=65mm]{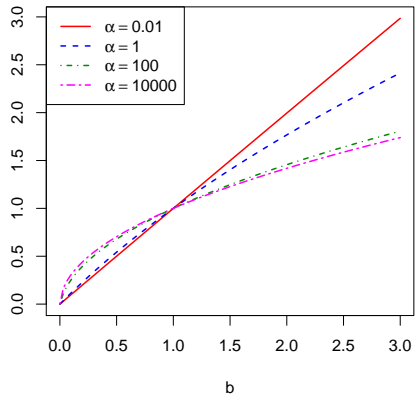}}
\subfigure[$\Phi(|b_1|; \alpha)+ \Phi(|b_2|; \alpha)\leq 1$]{\includegraphics[width=65mm,height=65mm]{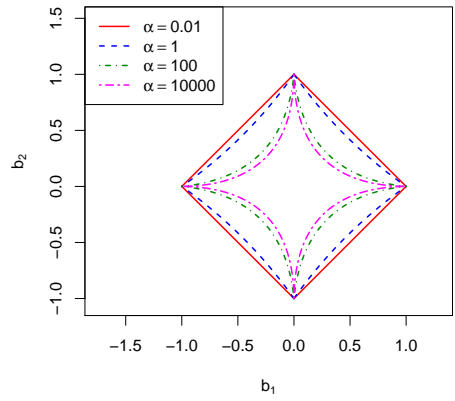}} \\
\subfigure[$\Phi(|b|^2; \alpha)$ vs. $|b|$]{\includegraphics[width=65mm,height=65mm]{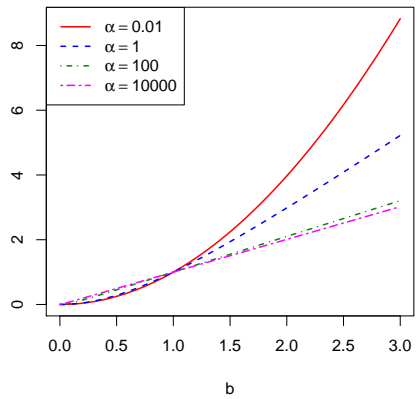}}
\subfigure[$\Phi(|b_1|^2; \alpha)+ \Phi(|b_2|^2; \alpha)\leq 1$]{\includegraphics[width=65mm,height=65mm]{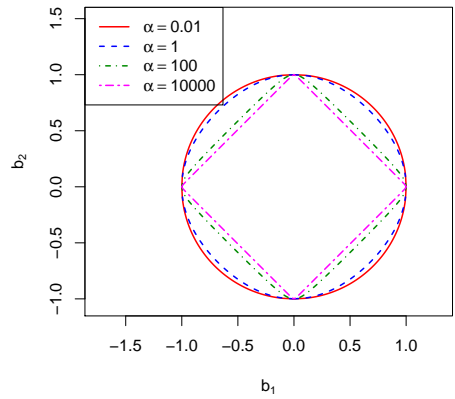}} \\
\caption{The KEP  functions  $\Phi(|b|; \alpha)$.} \label{fig:penalty}
\end{figure}

\section{Sparse Estimation Based on the KEP Penalty}
\label{sec:sest}

When $q=1$ the KEP function defines a nonconvex penalty for $b$ and is
singular at the origin. Thus, such a  penalty is able to induce
sparsity.
We now study the mathematical properties of  the sparse estimator in the settings  $q=1$.
These properties show that the KEP penalty is suitable for a coordinate descent algorithm~\citep{MazumderSparsenet:11}.

\subsection{Threshold Operators}
\label{sec:threshold}

Following  \cite{Fan01}, we  define the penalized least squares problem
\begin{equation} \label{eqn:general}
J_{1}(b):=\frac{1}{2} (z- b)^2 +  \Psi(|b|; \eta, \alpha),
\end{equation}
where $z=\x^T\y$.  \cite{Fan01} stated that a good penalty should result in an estimator with three properties.
(1) Unbiasedness: it is nearly unbiased when the true unknown parameter is large; (2) Sparsity: it is a thresholding rule, which
automatically sets small estimated coefficients to zero; (3) Continuity: it is continuous in data $z$ to avoid instability in model prediction.

According to the discussion in \cite{Fan01}, the resulting
estimator is nearly unbiased due to that $\Psi'(|b|) = \frac {\eta}{ \sqrt{2\alpha |b| +1}} \rightarrow 0$ as $|b|\rightarrow \infty$.
Note that
\[
\lim_{ |b|\rightarrow \infty} \; \frac{\eta}{  \sqrt{2\alpha |b| +1}} \Big/ {\frac{1} {2|b|^{1/2}}} = \frac{2\eta}{\sqrt{2\alpha}}.
\]
Thus, for the KEP penalty $\Psi(|b|)$ and the $\ell_{1/2}$-norm penalty $|b|^{1/2}$,
the convergence rates of their derivatives to zero are same.

As also stated in \cite{Fan01}, it suffices for the resulting estimator to be a thresholding rule  that the minimum of
the function $|b|+  \Psi'(|b|)$ is positive. Moreover, a sufficient and necessary condition for ``continuity" is
the minimum of
$|b|+  \Psi'(|b|)$ is attained at $0$. In fact, we have the following theorem.

\begin{theorem} \label{thm:sparsty} Consider the penalized least squares problem in (\ref{eqn:general}).
\begin{enumerate}
\item[\emph{(i)}]  If $\eta \geq  \frac{1}{\alpha}$, then the resulting estimator  is a thresholding rule; that is,
\[
\hat{b} = S_{\alpha}(z, \eta) := \left\{ \begin{array}{ll}
\frac{{\sgn}(z)}{2 \alpha} \kappa(|z|) & \textrm{ if } |z| >
\frac{3}{2 \alpha} ({\alpha \eta})^{\frac{2}{3}} {-} \frac{1}{2 \alpha}, \\
0 & \textrm{ if } |z| \leq \frac{3}{2 \alpha} ({\alpha \eta})^{\frac{2}{3}} {-} \frac{1}{2 \alpha}.
\end{array} \right.
\]
where
\[
\kappa(|z|) = \frac{4(2\alpha |z| {+} 1)}{3}\cos^2\Big[\frac{1}{3}
\arccos\big( {-} {\alpha \eta}  (\frac{3}{2\alpha|z| {+} 1})^{\frac{3}{2}}  \big) \Big] {-} 1.
\]
\item[\emph{(ii)}] If $\eta<  \frac{1}{\alpha}$, then the resulting estimator  is  defined as
\[
\hat{b} = S_{\alpha}(z, \eta) := \left\{ \begin{array}{ll}
\frac{{\sgn}(z)}{2 \alpha} \kappa(|z|) & \textrm{ if } |z| >
{\eta}, \\
0 & \textrm{ if } |z| \leq {\eta},
\end{array} \right.
\]
which is continuous in $z$.
\end{enumerate}
\end{theorem}

\paragraph{Remarks}
In both the cases, we always have $|\hat{b}|\leq |z|$. The
objective function $J_{1}(b)$ in (\ref{eqn:general}) is strictly convex in $b$ whenever $\eta\leq  \frac{1}{\alpha}$. Moreover, according to Lemma~\ref{lem:33} in Appendix~\ref{app:aa}, the estimator
$\hat{b}$ in both the cases is strictly increasing w.r.t.\ $|z|$, and $\hat{b}$ is Lipschitz continuous
when $\eta< \frac{1}{\alpha}$ (also see Lemma~\ref{lem:33}).

%\begin{figure}[!ht]
%\centering
%% \begin{tabular}{ccc}
%\hspace{-0.7cm}
%\subfigure[$\alpha=100$]{\includegraphics[width=50mm,height=40mm]{shrinkage1_alpha100.eps}}  \hspace{-0.7cm}
%\subfigure[$\alpha=1$]{\includegraphics[width=50mm,height=40mm]{shrinkage1_alpha1.eps}} \hspace{-0.7cm}
%\subfigure[$\alpha=0.01$]{\includegraphics[width=50mm,height=40mm]{shrinkage1_alpha01.eps}} \\
%\hspace{-0.7cm}
%\subfigure[$\alpha=100$ and $\gamma=\sqrt{\alpha{+}1}{-}1$]{\includegraphics[width=50mm,height=40mm]{shrinkage2_alpha100.eps}}  \hspace{-0.7cm}
%\subfigure[$\alpha=1$ and $\gamma=\sqrt{\alpha{+}1}{-}1$]{\includegraphics[width=50mm,height=40mm]{shrinkage2_alpha1.eps}} \hspace{-0.7cm}
%\subfigure[$\alpha=0.01$ and $\gamma=\sqrt{\alpha{+}1}{-}1$]{\includegraphics[width=50mm,height=40mm]{shrinkage2_alpha01.eps}} \\
%%\end{tabular}
%\caption{Threshold rules for  KEP w.r.t.\ different values of $\alpha$ and $\eta$.
%%where $\eta= \frac{\lambda \alpha}{\sqrt{1{+}2\alpha}-1} = \frac{\lambda}{2}(1 {+} \sqrt{1{+}2\alpha})$.
%In the fist row $\eta> \frac{1}{\alpha}$, and in the second row $\eta \leq \frac{1}{\alpha}$.}
%\label{fig:thresh2}
%\end{figure}

We now explore  connection of the thresholding operator (function) based on the KEP penalty with the soft thresholding operator based on Lasso
and the half thresholding operator based on the $\ell_{1/2}$-norm penalty~\citep{XuTNN:2012}. For this purpose,  in terms of Proposition~\ref{lem:1}
we let $\eta=\frac{\lambda \alpha}{ \sqrt{2\alpha{+}1}{-}1}$ where $\lambda>0$ does not rely on $\alpha$.
Obviously, $\frac{\alpha}{\sqrt{2\alpha{+}1}{-}1} =\frac{\sqrt{2\alpha{+}1}{+}1}{2}\geq 1$.
Hence, $|z|\geq \frac{\lambda \alpha}{ \sqrt{2\alpha{+}1}{-}1}$ implies $|z|\geq \lambda$. Moreover, $\frac{\sqrt{2\alpha{+}1}{+}1}{2}$ is increasing but $\frac{1}{\sqrt{2\alpha{+}1}{-}1 }$
is decreasing in $\alpha$. This implies that the KEP penalty ($q=1$) to some extent satisfies the nesting property (see Figure~\ref{fig:thresh0}-(a)), a
desirable property for thresholding functions pointed out by~\cite{MazumderSparsenet:11}.
%has stronger sparseness  than the $\ell_1$-norm when $\eta\leq  \frac{4 \gamma}{\alpha^2}$.

Furthermore,
we have $\liml_{\alpha  \rightarrow 0} \frac{1}{\alpha} =\infty$ and  $\liml_{\alpha  \rightarrow 0} \frac{\lambda \alpha}{ \sqrt{2\alpha{+}1}{-}1} =\lambda$.
In this limiting case, it is clear that our thresholding function  approaches the soft thresholding function:
\[
\lim_{\alpha \to 0+} S_{\alpha}(z, \eta)= S(z, \lambda) :=\sgn(z)(|z|-\lambda)_{+}=\left\{\begin{array}{ll} 0 & \mbox{ if } |z|\leq \lambda, \\
\sgn(z)(|z|-\eta) &   \mbox{ if } |z|> \lambda. \end{array}  \right.
\]

Next, we take the limits that $\liml_{\alpha  \rightarrow \infty} \frac{1}{\alpha} =0$,
$\liml_{\alpha  \rightarrow \infty} \frac{3}{2\alpha} \Big(\frac{\alpha^2 \lambda}{\sqrt{2\alpha{+}1}-1}\Big)^{\frac{2}{3}} {-} \frac{1}{2\alpha} = 3
(\frac{\lambda}{4})^{2/3}$, and $\liml_{\alpha  \rightarrow \infty} \frac{\alpha^2 \lambda}{\sqrt{2\alpha{+}1}-1} (\frac{3}{2\alpha|z| {+} 1})^{\frac{3}{2}} =\frac{\lambda}{4} (\frac{3}{|z|})^{\frac{3}{2}}$. In this limiting case, $\eta>  \frac{1}{\alpha}$ is always met.
Thus, the  resulting estimator in Theorem~\ref{thm:sparsty}-(i)   degenerates to
\[
\hat{b} =  S_{\frac{1}{2}}(z, \lambda): = \left\{ \begin{array}{ll}
{\sgn}(z) \frac{4|z|}{3}\cos^2\Big[\frac{1}{3}
\arccos\big( {-}\frac{\lambda}{4} (\frac{3}{|z|})^{\frac{3}{2}}  \big) \Big]  & \textrm{ if } |z| >
3 (\frac{\lambda}{4})^{2/3}, \\
0 & \textrm{ if } |z| \leq 3 (\frac{\lambda}{4})^{2/3},
\end{array} \right.
\]
which is well established by \cite{XuTNN:2012}. Obviously, the above thresholding function is not continuous at $|z|=3 (\frac{\lambda}{4})^{2/3}$.
However,  the KEP penalty ($q=1$) can make the resulting estimators have ``unbiasedness," ``sparsity" and ``continuality"
by assuming $\eta\leq  \frac{1}{\alpha}$. Moreover, the KEP penalty satisfies the nesting property.

The previous analysis implies that $\alpha>0$ plays a role of ``temperature" in statistical physics.
When $\alpha \to \infty$, the thresholding function becomes  discontinuous from continuous status, yielding a ``phase transition" phenomenon.

In Figure~\ref{fig:thresh0}-(b) we compare the
thresholding rules for  the hard ($\ell_0$), soft ($\ell_1$ or Lasso), half ($\ell_{1/2}$) and the KEP penalty.
In Section~\ref{sec:related} we explore the relationship between KEP and MCP as well as the relationship between the thresholding functions based on KEP and MCP.

\begin{figure}[!ht]
\centering
%% \begin{tabular}{ccc}
%%\hspace{-1.5cm}
\subfigure[$\lambda=\frac{1}{4}$]{\includegraphics[width=75mm,height=50mm]{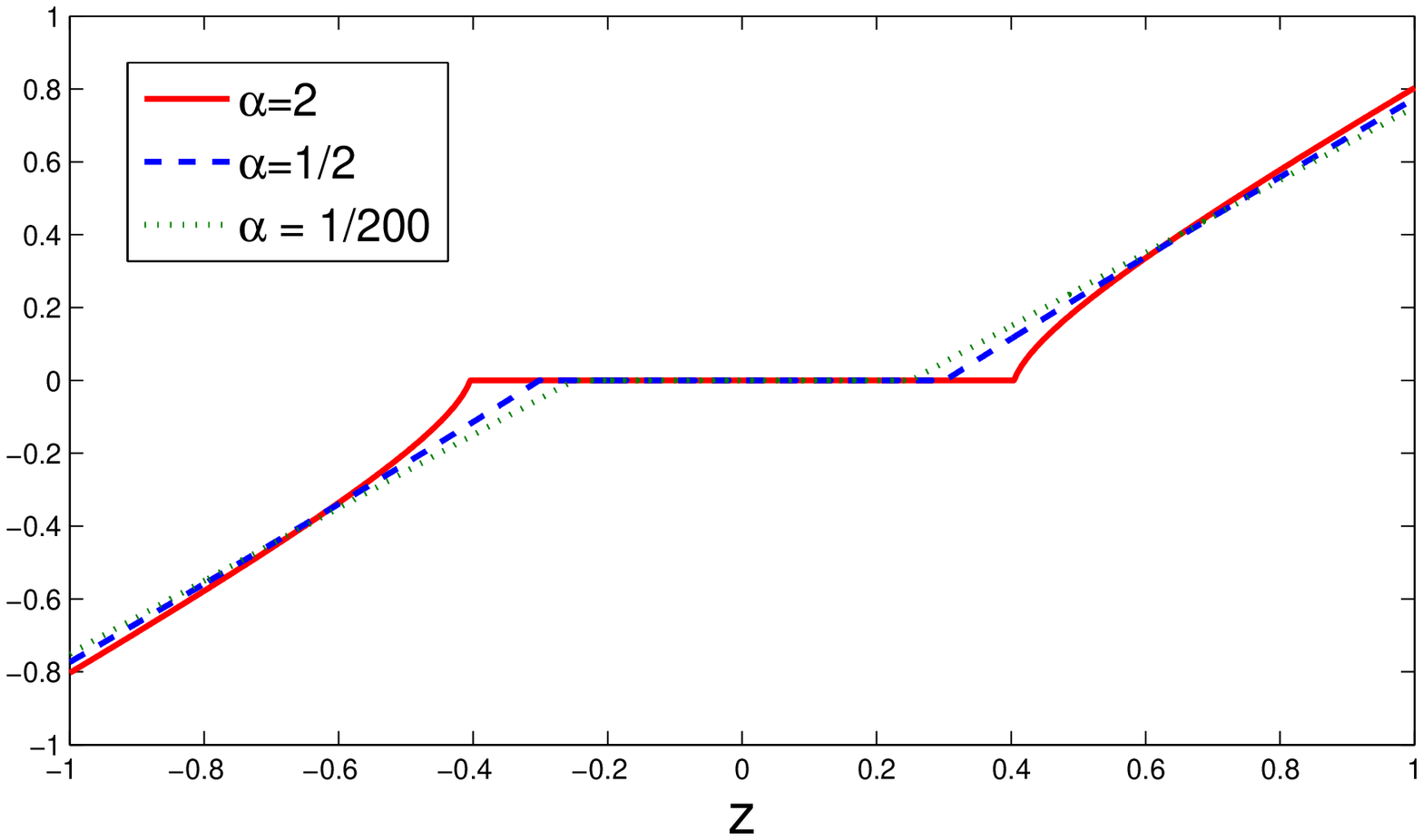}} \subfigure[$\alpha=\frac{1}{2}$ and $\lambda=1$]{\includegraphics[width=75mm,height=50mm]{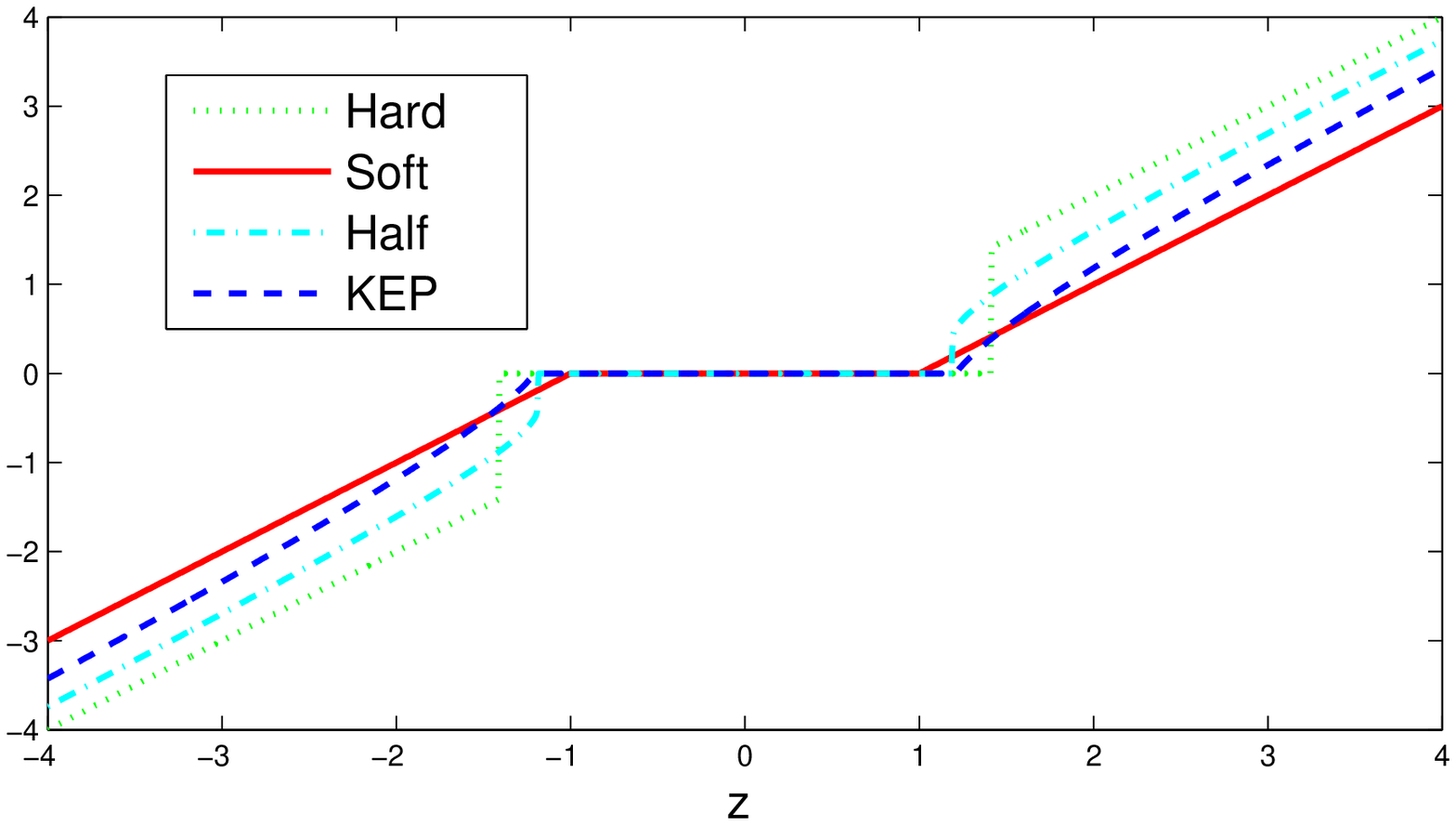}}  \\
%\end{tabular}
\caption{ (a) Threshold rules for KEP with $\eta=\frac{\lambda }{2} ({\sqrt{1{+}2\alpha}{+}1} )$ where  $\lambda$ is fixed and $\alpha$ varies. (b) Threshold rules for the Hard ($\ell_0$) ($z I(|z|\geq \sqrt{2 \lambda})$), Soft ($\ell_1$) ($\sgn(z)(|z|-\lambda)_{+}$), Half ($\ell_{1/2}$) and  KEP with $\eta=\frac{\lambda }{2} ({\sqrt{1{+}2\alpha}{+}1} )$. }
\label{fig:thresh0}
\end{figure}

%\subfigure[$\Phi(|b|; \alpha)$ vs. $|b|$]{\includegraphics[width=65mm,height=65mm]{q1_4.pdf}}

\subsection{The Coordinate Descent Algorithm}
\label{sec:cda}

Given the training dataset $(\y, \X)$, we consider the following minimization problem
\[
J(\b) = \frac{1}{2}\|\y- \X \b\|_2^2 + \sum_{j=1}^p \Psi(|b_j|; \eta, \alpha).
\]
Based on the discussion in the previous subsection, the KEP penalty with $q=1$ is suitable
for the coordinate descent algorithm. Particularly,
the coordinate descent procedure of solving the above minimization problem is given Algorithm~\ref{alg:coord}.

Obviously, $\Psi(|b|; \eta, \alpha)$ is symmetric around $0$. Moreover, $\Psi'(|b|; \eta, \alpha)=\eta (1+2 \alpha |b|)^{-1/2}$ (the derivative w.r.t.\ $|b|$) is positive, continuous and uniformly
bounded (i.e., $\Psi'(|b|; \eta, \alpha) \leq \eta$) for $|b|\geq 0$. Since
\[
\Psi{''}(|b|; \eta, \alpha)= \frac{d^2 \Psi(|b|; \eta, \alpha)}{d |b|^2}=- \eta \alpha (1+2 \alpha |b|)^{-3/2},
\]
we have that $\minl_{b} \; \Psi{''}(|b|; \eta, \alpha) =- \eta \alpha >-1$ when $\eta < \frac{1}{ \alpha}$.

Assume $(\y, \X)$ lies on a compact set and no column of $\X$ is degenerate. It then follows from Theorem~4 of \cite{MazumderSparsenet:11}
that the univariate maps $b \mapsto J_1(b)$ are strictly convex and that the sequence $\{\b^{(t)}; t=1, 2, \cdots\}$ generated via Algorithm~\ref{alg:coord} converges to a (local) minimum of the objective function $J(\b)$.

%the algorithm enjoys the same  convergence property proposed by .

\begin{algorithm}[!ht]
   \caption{The coordinate descent algorithm}
   \label{alg:coord}
\begin{algorithmic}
   \STATE {\bfseries Input:} $\{\x_i, y_i\}_{i=1}^n$ where each column of $\X=[\x_i, \ldots, \x_n]^T$ is standardized
   to have mean 0 and length 1,
   a grid of increasing values $\Lambda=\{\lambda_1, \ldots, \lambda_L\}$, a grid of decreasing values $\Gamma=\{\alpha_1, \ldots, \alpha_K\}$
   where $\alpha_K$ indexes the Lasso penalty.
   \STATE Set  $\hat{\b}_{\alpha_K, \eta_{L+1}}=0$.
  % \STATE For each value of $l \in \{L, L-1, \ldots, 1\}$
   \FOR{each value of $l \in \{L, L-1, \ldots, 1\}$ }
 %  \REPEAT
   \STATE Initialize  $\tilde{\b} = \hat{\b}_{\alpha_K, \eta_{l+1}}$;
    \FOR{each value of $k \in \{K, K-1, \ldots, 1\}$ }
    \STATE Compute $\eta_{lk}= \frac{\lambda_l }{2}(1+\sqrt{1+2 \alpha_k})$
    \IF{$\eta_{lk} {\alpha_k} < 1$ }
  %   \REPEAT
     \STATE Cycle through the following one-at-a-time updates
      \[\tilde{b}_j = S_{\alpha_{k}} \Big(\sum_{i=1}^n(y_i- z_{i}^j)x_{ij}, \eta_{lk}\Big), \quad j=1, \ldots, p
       \]
      where $z_i^j=\sum_{k\neq j} x_{ik} \tilde{b}_k$, until the updates converge to $\b^{\ast}$;
   \STATE $\hat{\b}_{\alpha_k, \lambda_l} \leftarrow \b^{\ast}$.
     %  \UNTIL
   \ENDIF
   \ENDFOR
   \STATE Increment $k$;
 %     \UNTIL
   \ENDFOR
   \STATE Decrement $l$;
   \STATE {\bfseries Output:} Return the two-dimensional solution $\hat{\b}_{\alpha, \lambda}$ for $(\alpha, \lambda) \in \Gamma{\times}\Lambda$.
\end{algorithmic}
\end{algorithm}

Note that the second-order derivative of $\lambda |b|^{1/2}$ w.r.t.\ $|b|$ is $-\frac{\lambda}{4 |b|^{3/2}}$ and $\infl_{b} -\frac{\lambda}{4 |b|^{3/2}}=-\infty$ for a fixed positive $\lambda$. Thus, the convergence result given in Theorem~4 of \cite{MazumderSparsenet:11}
is not applicable  to the $\ell_{1/2}$-penalty case.

It is worth pointing out that the iteratively reweighted $\ell_1$ method of \citet{Daubechies:2010}
is essentially equivalent to the multi-state LLA procedure  of  \citet{TZhangJMLR:10}.
The multi-state LLA for  the minimization problem in (\ref{eqn:general}) gives the following  update
\[
b^{(t+1)}=S(z, w_0^{(t)})= \argmin_{b } \frac{1}{2} (b-z)^2+ w_0^{(t)} |b|
\]
where $w_0^{(t)}= \eta/\sqrt{1+2\alpha |b^{(t)}|}$, i.e., the derivative of $\Psi(|b|;\eta, \alpha)$ at $|b|=|b^{(t)}|$.
Since $J_1(b)$ is strictly convex when $\eta \alpha<1$, it is also reasonable to let $\eta \alpha<1$ when applying the multi-state LLA method.

For the sake of simplicity, we assume that $z\geq 0$.
If $z \geq \eta/\sqrt{1+2\alpha |b^{(i)}|}$
for any $1 \leq i \leq k$, we obtain $b^{(t)}\geq 0$. Using the fact that $1/\sqrt{1+2\alpha s}$ is convex in $s\geq 0$,
we have
\[
b^{(t+1)}=z- \eta/\sqrt{1+2\alpha |b^{(t)}|}\leq z- \eta + \eta \alpha b^{(t)}\leq (z- \eta)\sum_{i=0}^t (\eta\alpha)^i
+ (\eta\alpha)^{(t+1)} b^{(1)},
\]
which implies that the multi-state LLA procedure converges to the minimum of $J_1(b)$ ($\eta \alpha<1$)
at rate $O((\eta \alpha)^t)$ in the worst case. This result agrees with that of \citet{MazumderSparsenet:11} about the univariate MCP  penalized squares problem.  As a result, the number of iterations required for the multi-state LLA procedure to converge with an $\epsilon$ tolerance of
the minimizer of $J_1(b)$ is of order $-\frac{\log (\epsilon)}{\log (\eta \alpha)}$.
Thus, the multi-state LLA based coordinate-wise  method is less efficient than Algorithm~\ref{alg:coord}.

\section{Relationships Between KEP and MCP}
\label{sec:related}

In Sections~\ref{sec:method0} and \ref{sec:sest}  we discuss the relationship of KEP with the $\ell_{1}/2$ and $\ell_{1}$ norms.
In this section we explore the relationship between KEP and MCP.

Note that $\chi^2$ distance $\frac{(w-\eta)^2}{w}$ between $w$ and $\eta$ is not symmetric. Thus, it is also interesting to
consider the concave conjugate of $\frac{(w-\eta)^2}{\eta}$. In this regard,
the corresponding concave conjugate is given by
\[
 \min_{w\geq 0} \Big\{w s +\frac{1}{2 \alpha} \frac{(w-\eta)^2}{\eta}  \Big\}.
\]
We denote the minimum as $\eta M(s)$ where
\[
 M(s) = \left\{\begin{array}{ll} \frac{1}{2 \alpha} & \mbox{ if } s\geq \frac{1}{\alpha}, \\
s - \frac{\alpha s^2}{2} & \mbox{ if } s < \frac{1}{\alpha},  \end{array} \right.
\]
which is in fact the MCP function of \citet{MazumderSparsenet:11} when setting $\frac{1}{\alpha}=\lambda \gamma$ and $\eta=\lambda$ therein.
This recovers an important connection between KEP and MCP; that is,  both are  based on the $\chi^2$-distance.
Note that \citet{ZhangAAAI:2013} constructed the MCP function using the concave conjugate of the squared Euclidean distance function. Their construction approach is essentially
equivalent to the previous construction, because $|w-\eta|^2$ is the squared Euclidean distance  and $\alpha \eta$
can be treated as a new single parameter.

Let us return to
the KEP function $\Psi(s; \eta, \alpha)$ defined in (\ref{eqn:relativity}) where $q=1$ and $s=|b|$.  Furthermore, we define $\Psi(s; \eta, \alpha) = \eta K(s)$ where  $K(s)= \frac{1}{\alpha} (\sqrt{2 \alpha s +1}-1)$.
For a fixed $\alpha>0$, it is easily verified that
\[
M(s)\leq K(s)\leq s,
\]
with equality only if $s=0$ (also see Figure~\ref{fig:thresh1}(a)). Additionally, $K(s)$ is infinitely differentiable on $[0, \infty)$.
However, $M(s)$ is only first-order differentiable on $[0, \infty)$. The second-order derivative of $M(s)$ at
$s=\1/\alpha$ does not exist (see Figure~\ref{fig:thresh1}(b)). However, the convergence result of \citet{MazumderSparsenet:11}
is built on the assumption that the second-order derivative exists (see Theorem 4 therein).

\begin{figure}[!ht]
\centering
%% \begin{tabular}{ccc}
%%\hspace{-1.5cm}
\subfigure[$\alpha=1$]{\includegraphics[width=75mm,height=50mm]{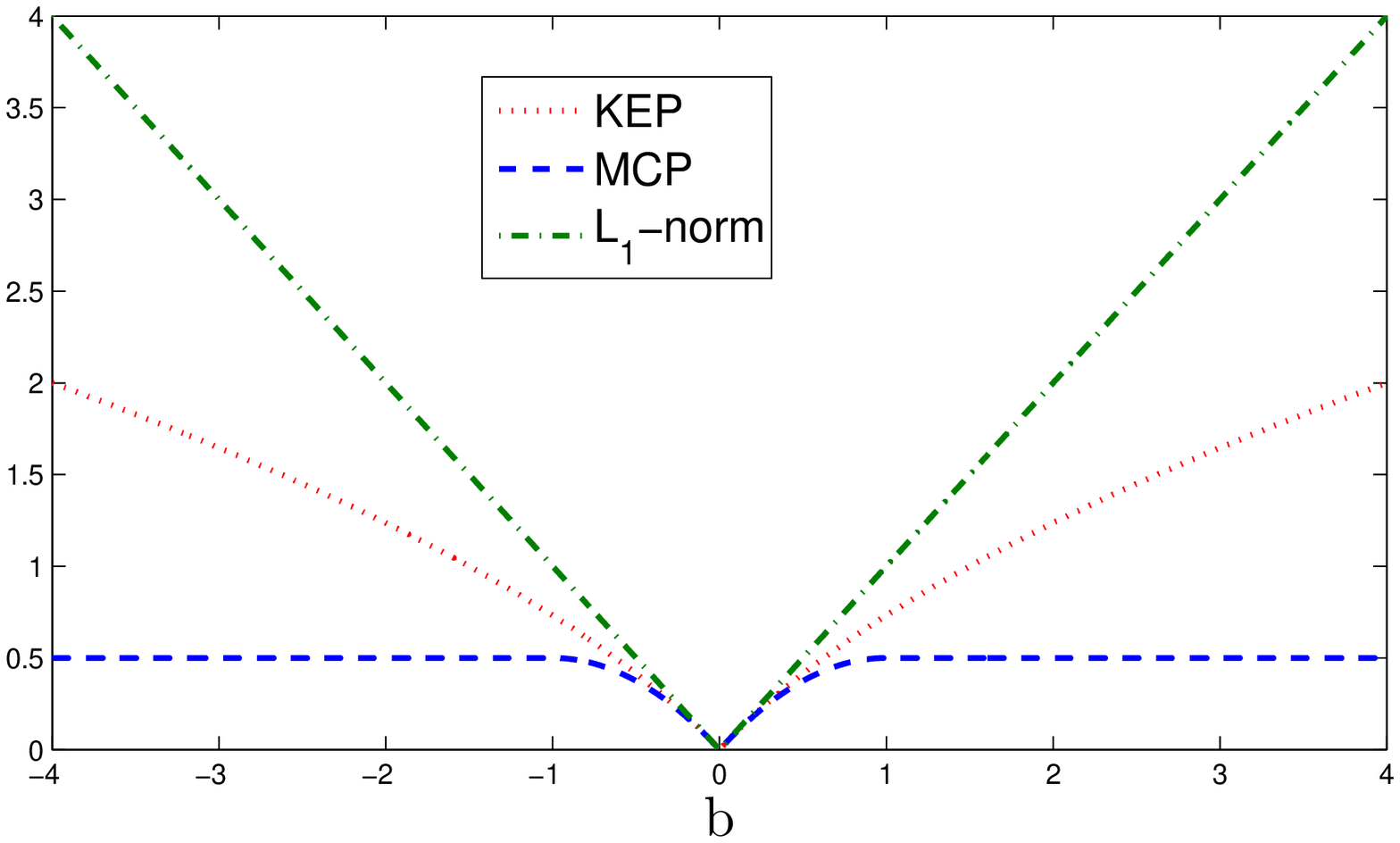}} \subfigure[$\alpha=1$]{\includegraphics[width=75mm,height=50mm]{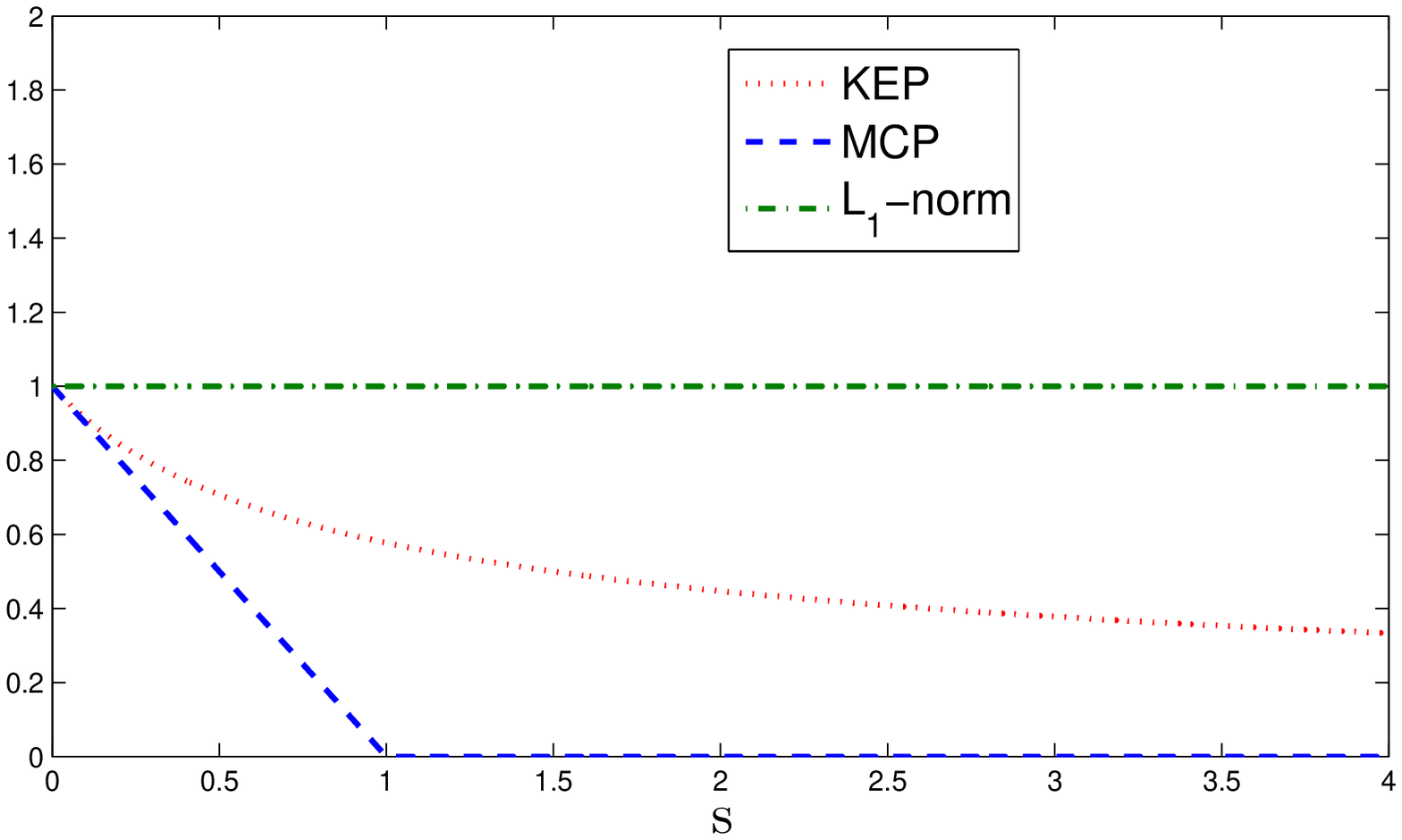}} \\
%\end{tabular}
\caption{(a) The functions: KEP $K(|b|)$, MCP $M(|b|)$ and  $\ell_1$-norm $|b|$ w.r.t.\ $b \in (-\infty, \infty)$.  (b) The derivatives of $K(s)$, $M(s)$ and $s$ w.r.t\ $s\geq 0$.}
\label{fig:thresh1}
\end{figure}

To obtain
the thresholding function w.r.t.\ MCP, we also need to consider the two cases that $\eta \alpha<1$ and $\eta \alpha \geq  1$. In the first case that  $\eta \alpha<1$, the thresholding function
is given as
\[
S_{\alpha}(z, \eta) =\left\{\begin{array}{ll} 0 & \mbox{ if } |z|\leq \eta \\  \sgn(z) \frac{|z| -\eta} {1- \alpha \eta} & \mbox{ if } \eta < |z|\leq  \frac{1}{\alpha} \\  z & \mbox{ if } |z|> \frac{1}{\alpha},
 \end{array} \right.
\]
which is identical to the one of \citet{MazumderSparsenet:11} when setting $\frac{1}{\alpha}=\lambda \gamma$ and $\eta=\lambda$.
The resulting rule $S_{\alpha}(z, \eta)$ is obviously continuous.
However, $S_{\alpha}(z, \eta)$ is not smooth for $|z|>\eta$. Specifically,  $S_{\alpha}(z, \eta)$  is not differentiable at $|z|=1/\alpha$.
Recall that the thresholding function w.r.t.\ KEP is always smooth for $|z|>\eta$  in the case that  $\eta \alpha<1$ (see Theorem~\ref{thm:sparsty}).  In Figure~\ref{fig:thresh}-(a), we illustrate comparison of MCP with the $\ell_1$-norm and KEP.
As we see, KEP can be treated as a trade-off of the $\ell_1$-norm and MCP in unbiasedness and  differentiability.

In the second case that $\eta \alpha \geq  1$,  the thresholding function w.r.t.\ MCP  is
\[
H_{\alpha}(z, \eta) =\left\{\begin{array}{ll} 0 & \mbox{ if } |z|\leq \frac{1}{\alpha} \\   z & \mbox{ if } |z|> \frac{1}{\alpha}.
 \end{array} \right.
\]
The derivation is based on some direct computations, so we omit it.
Clearly, $H_{\alpha}(z, \eta)$ is not continuous at $|z|=\frac{1}{\alpha}$ in this case  (see Figure~\ref{fig:thresh}-(b)).
Especially, when $\alpha \eta =1$, the thresholding function is also  not continuous at $|z|=\frac{1}{\alpha}$.
However,  it  is obtained from Theorem~\ref{thm:sparsty} that the
thresholding function w.r.t.\ KEP is still continuous when $\alpha \eta =1$ (see Figure~\ref{fig:thresh}-(c)).
%becomes
%\[
%S_{\alpha}(z, \eta) =\left\{\begin{array}{ll} 0 & \mbox{ if } |z|\leq \eta \\   z & \mbox{ if } |z|> \eta.
% \end{array} \right.
%\]

\begin{figure}[!ht]
\centering
%% \begin{tabular}{ccc}
%%\hspace{-1.5cm}
\subfigure[$\eta=1$ and $\alpha=1/2$]{\includegraphics[width=50mm,height=40mm]{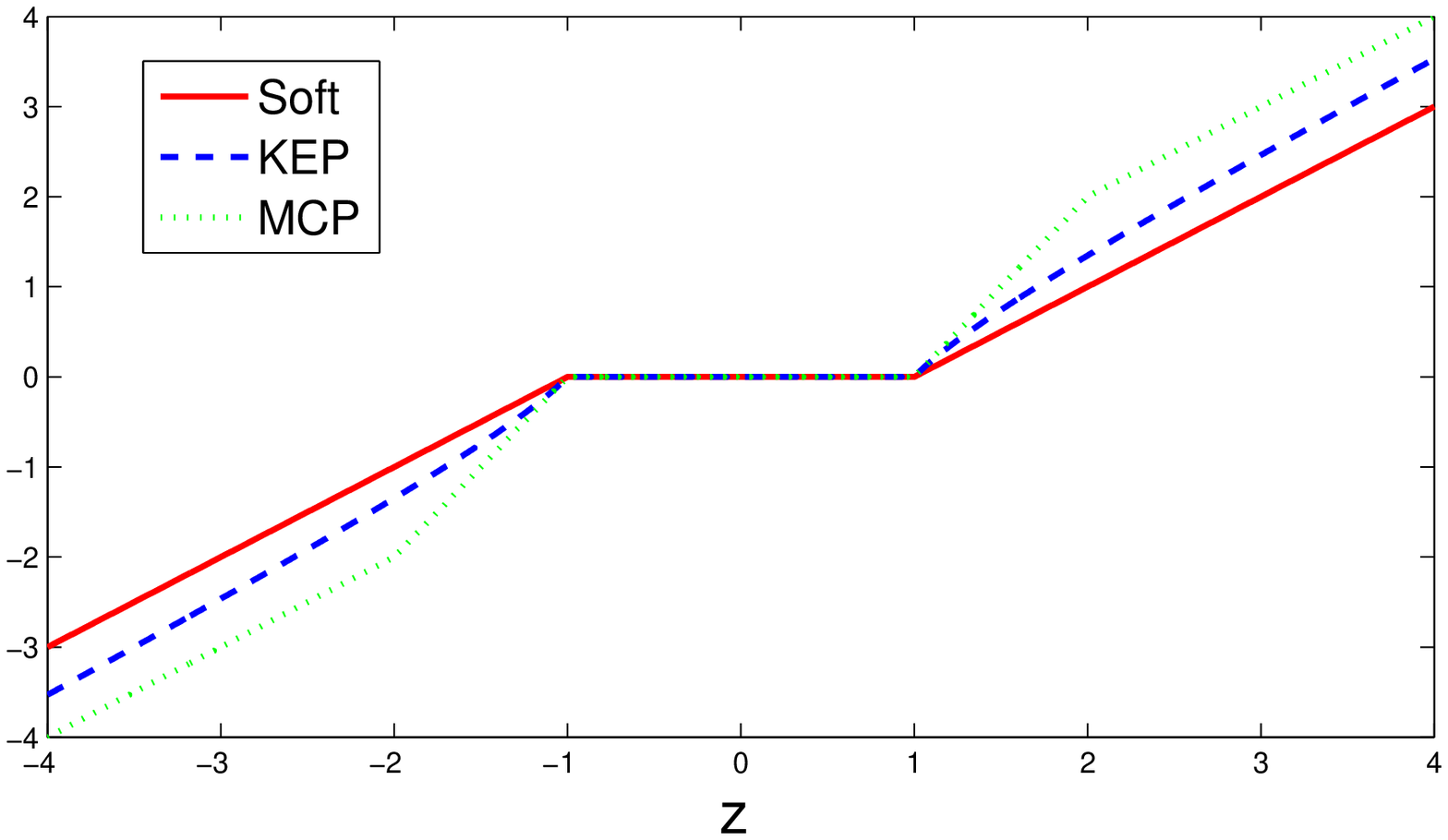}}  %\vsace{-0.1in}
\subfigure[$\alpha=2$ and $\eta=1$]{\includegraphics[width=50mm,height=40mm]{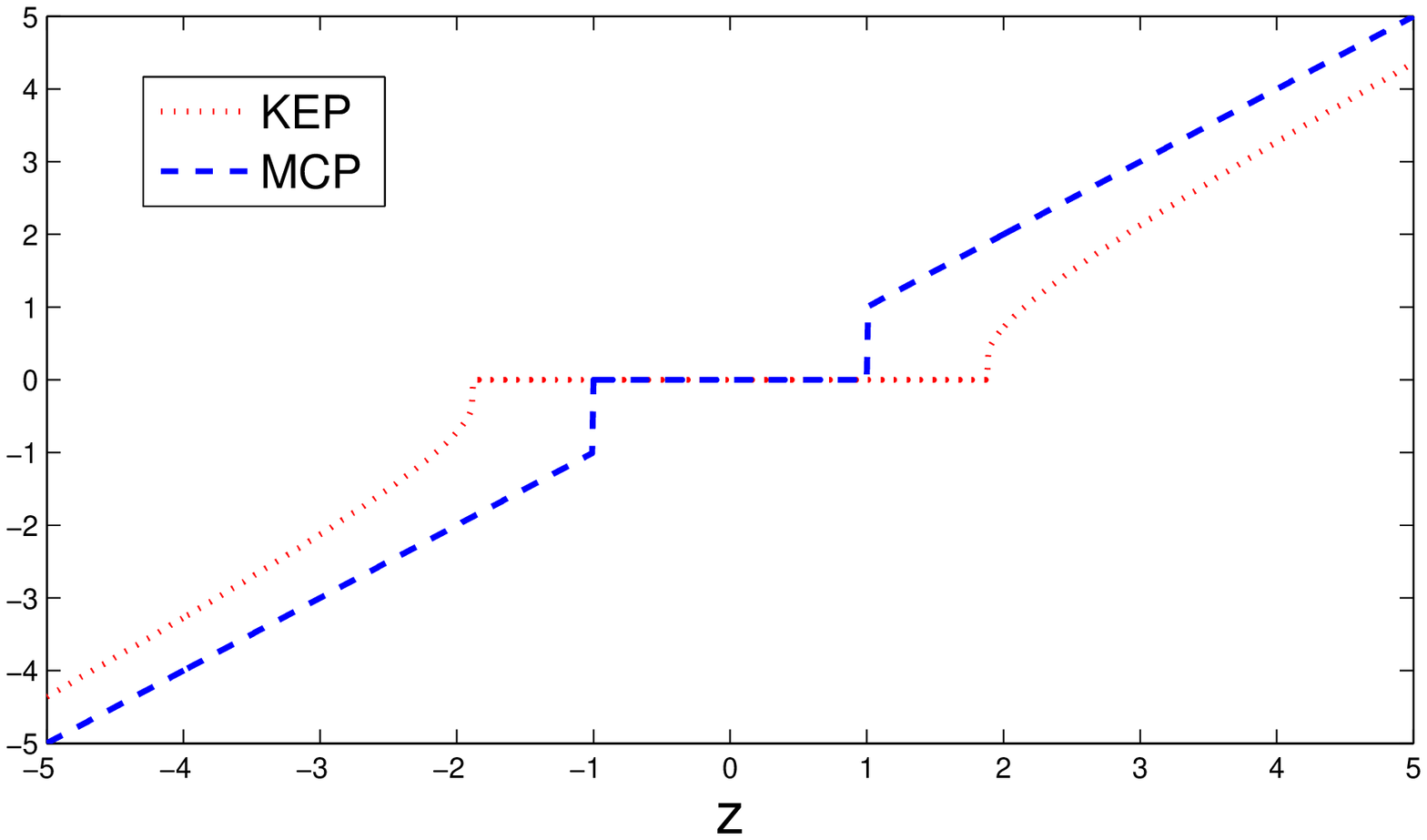}} %\vspace{-0.1in}
\subfigure[$\alpha=1$ and $\eta=1$]{\includegraphics[width=50mm,height=40mm]{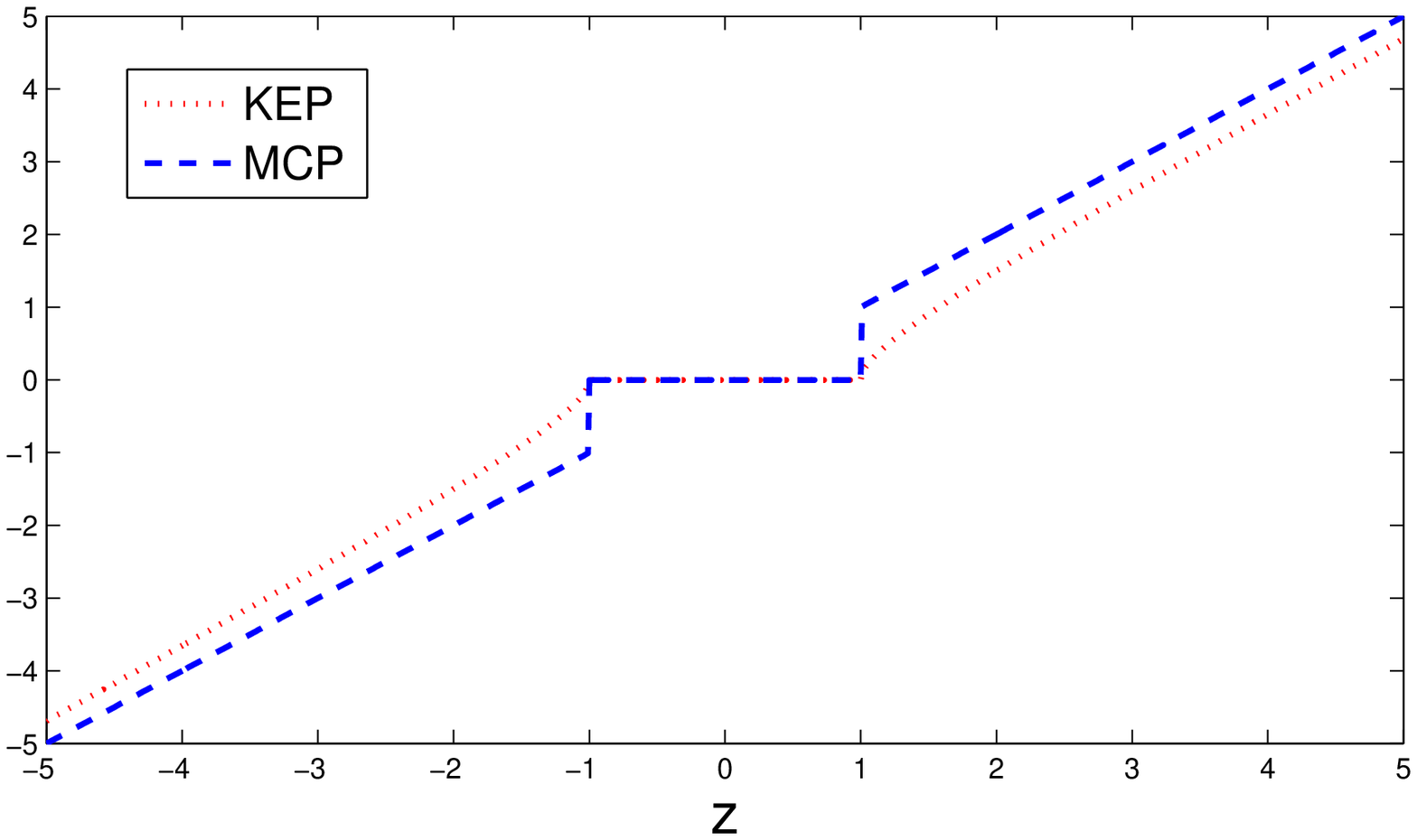}}  \\
%\end{tabular}
\caption{(a) Threshold rules for the  Soft (Lasso) ($\sgn(z)(|z|-\eta)_{+}$), KEP and MCP under the condition $\eta \alpha>1$; (b) KEP and MCP under the condition $\eta \alpha>1$;  (c) KEP and MCP under the condition $\eta \alpha=1$.}
\label{fig:thresh}
\end{figure}

%It is worth pointing out that the iteratively reweighted $\ell_1$ method
%is essentially equivalent to an LLA procedure  of  \citet{ZouLi:2008,TZhangJMLR:10}.
%The multi-state LLA for the univariate MCP  penalized squares problem gives the following  updates
%\[
%b^{(t+1)}=S(z, \omega^{(t)})= \argmin_{b } \frac{1}{2} (b-z)^2+ \omega^{(t)} |b|
%\]
%where $\omega^{(t)}= \eta(1-\alpha |b^{(t)}|)_{+}$, i.e., the derivative of $\eta M(|b|)$ at $|b|=|b^{(t)}|$.
%\citet{MazumderSparsenet:11} showed that when $\eta \alpha>1$ this procedure converges to the minimum of $\frac{1}{2} (b-z)^2+ \eta M(|b|)$
%at (worst case) rate $O(1/(\eta \alpha)^t)$.
%
%For the minimization problem in (\ref{eqn:general}) based on $\Psi(|b|;\eta, \alpha)$,
%The multi-state LLA gives then the following  updates
%\[
%b^{(t+1)}=S(z, w_0^{(t)})= \argmin_{b } \frac{1}{2} (b-z)^2+ w_0^{(t)} |b|
%\]
%where $w_0^{(t)}= \eta/\sqrt{1+2\alpha |b^{(t)}|}$, i.e., the derivative of $\Psi(|b|;\eta, \alpha)$ at $|b|=|b^{(t)}|$.
%For the sake of simplicity, we assume that $z\geq 0$.
%If $z \geq \eta/\sqrt{1+2\alpha |b^{(i)}|}$
%for any $1 \leq i \leq k$, we obtain $|b^{(t)}\geq 0$. Thus,
%\[
%b^{(t+1)}=z- \eta/\sqrt{1+2\alpha |b^{(t)}|}\leq z- \eta + \eta \alpha b^{(t)}\leq (z- \eta)\sum_{i=0}^t (\eta\alpha)^i + (\eta\alpha)^{(t+1)}
%b^{(1)},
%\]
%which implies that the multi-state LLA procedure converges to the minimum of $J_1(b)$ also
%at (worst case) rate $O(1/(\eta \alpha)^t)$.

We now take  behaviours as $\alpha$ approaches to limiting cases. First, we immediately have  that $\liml_{\alpha \to 0+} \; M(s) = s$ and
\[
\lim_{\alpha \to 0+} S_{\alpha}(z, \eta) =S(z, \eta) :=\left\{\begin{array}{ll} 0 & \mbox{ if } |z| \leq \eta \\  \sgn(z)( |z|-\eta) & \mbox{ if } |z|> \eta.
 \end{array} \right.
\]
Second, let $\eta=\frac{\lambda}{M(1)}$ where $\lambda>0$ is a constant that independents on $\alpha$. We have that
\[
\lim_{\alpha \to 0+} \; \frac{M(s)}{M(1)} = s \quad \mbox{and} \quad \lim_{\alpha \to \infty} \; \frac{M(s)}{M(1)} =\left\{\begin{array}{ll} 0 & \mbox{ if } s=0 \\ 1 & \mbox{ if } s\neq 0. \end{array} \right.
\]
This shows that $\frac{M(s)}{M(1)}$ get the entire continuum from the $\ell_1$-norm to the $\ell_0$-norm, as varying from $\alpha \to 0+$ to $\alpha \to \infty$.
However, it is not tractable to derive the  thresholding function corresponding to the penalty function $\frac{M(s)}{M(1)}$ because $M(1)$  as a function of $\alpha$ is not smooth.
We feel that this would be an important reason that MCP does not hold the nesting property~\citep{MazumderSparsenet:11}. In contrast,
KEP can keep this property by setting $\eta= \frac{\lambda \alpha}{\sqrt{1{+}2\alpha} {-} 1}$.

When we let $s=|b|^2$,  $K(|b|^2)$ is convex in $|b|$ (see Section~\ref{sec:method0}). In fact, $K(|b|^2)$  is used
by \citet{Daubechies:2010} in devising the iterative reweighted $\ell_2$ method (see Section~\ref{sec:method0}). However,
$M(|b|^2)$ is neither convex nor concave in $|b|$.
%An other  important advantage of KEP over MCP is the ability of  KEP in  Bayesian modeling.
%\citet{ZouLi:2008} proved that the MCP function does not have Bayesian interpretations. In Section~\ref{sec:bayes}, we
%show that the KEP function is able to induce a prior distribution.

\section{Asymptotic Properties}
\label{sec:asymp}

We discuss asymptotic properties of sparse estimators. Following the setup of \cite{ZouLi:2008},
we assume two conditions: (1) $y_i=\x_i^T \b^{*} + \epsilon_i$ where $\epsilon_1, \ldots, \epsilon_n$ are iid errors
with mean 0 and variance $\sigma^2$;
(2) $\X^T \X/n\rightarrow \C$ where $\C$ is a positive definite matrix. Let ${\cal A}=\{j: b_{j}^{*} \neq 0\}$.
Without loss of generality, we assume that ${\mathcal A}=\{1, 2, \ldots, r\}$ with $r<p$. Thus, partition $\C$ as
\[
\begin{bmatrix}\C_{11} & \C_{12} \\  \C_{21} & \C_{22} \end{bmatrix},
\]
where $\C_{11}$ is $r{\times} r$. Additionally,  let $\b^{*}_{1} =\{b^{*}_{j}: j \in {\mathcal A}\}$
and $\b^{*}_{2} = \{b^{*}_{j}:  j \notin {\cal A}\}$.

Recall that the iteratively reweighted $\ell_1$ method of \cite{Daubechies:2010} can be regarded as a multi-stage LLA estimator~\citep{TZhangJMLR:10}. Specifically,
we study the oracle property of  the one-step LLA suggested by \citet{ZouLi:2008}. Based on the KEP with $q=1$,
we consider  the following one-step sparse estimator:
\[
\b_n^{(1)}=\argmin_{\b}  \;  \|\y {-} \X \b\|_2^2  +   \sum_{j=1}^p    \frac{\eta_n}{\sqrt{1{+} 2\alpha_n |b^{(0)}_j|}} |b_j|,
\]
where $\b^{(0)}=\big(b_1^{(0)}, \ldots, b_p^{(0)}\big)^T $ is a root-$n$-consistent estimator to $\b^{*}$.  The following theorem shows that
this estimator has the oracle property. That is,
\begin{theorem}  \label{thm:oracle1} Let $\b_{n1}^{(1)}=\{b_{nj}^{(1)}: j \in \AM\}$  and ${\cal A}^{(1)}_n=\{j: b_{nj}^{(1)} \neq 0\}$.
Suppose that
$n^{-3/4} \eta_n \rightarrow 0$, $\eta_n/\sqrt{n} \rightarrow \infty$, and $\alpha_n/\sqrt{n} \rightarrow c_1$ where $c_1 \in (0, \infty)$.
Then $\b_n^{(1)}$ satisfies the following properties:
\begin{enumerate}
\item[\emph{(1)}]  Consistency in variable selection:   \[\lim_{n \rightarrow \infty} P({\cal A}^{(1)}_n={\cal A})=1.\]
\item[\emph{(2)}]  Asymptotic normality: \[\sqrt{n}(\b^{(1)}_{n1} - \b^{*}_{1}) \overset{d}{\longrightarrow} N(\0, \sigma^2 \C_{11}^{-1}).\]
\end{enumerate}
\end{theorem}

As we mentioned earlier,  $2\alpha=1/\epsilon^2$ and $2 \eta=\lambda/\epsilon$ in the iteratively reweighted $\ell_1$ method of \cite{Daubechies:2010}.
In this case, we make the assumption that $ \lambda_n/ n^{1/4} \rightarrow \infty$, $ \lambda_n/\sqrt{n} \rightarrow 0$,  $\epsilon^{2} = O(n^{-1/2})$  (or $\epsilon^{-2} n^{-1/2} \rightarrow c_1 \in (0, \infty)$).
Then the resulting estimators have the oracle properties. Recall that~\cite{Daubechies:2010} set $\epsilon^{(t{+}1)} = \min(\epsilon^{(t)}, r_k(\b^{(t)})/p)$. This makes it sense that  $\epsilon^{-2} n^{-1/2} \rightarrow c_1$.

Let us return to the sparse estimator based on the penalty function $\Psi(|b|;\eta, \alpha)$ itself. That is,
\begin{equation} \label{eqn:sqre}
\tilde{\b}_n=\argmin_{\b}  \;  \|\y {-} \X \b\|_2^2 +   \frac{\eta_n} {\alpha_n} \sum_{j=1}^p   \Big[\sqrt{1{+} 2\alpha_n |b_j|} {-}1\Big].
\end{equation}

%For this estimator, we have the following property.
\begin{theorem}  \label{thm:oracle2} Let $\tilde{\b}_{n1}=\{\tilde{b}_{nj}: j \in \AM\}$  and $\tilde{ {\cal A}}_n=\{j: \tilde{b}_{nj} \neq 0\}$. Suppose that
$\eta_n/n^{3/4}  \rightarrow 0$, $\eta_n/\sqrt{n} \rightarrow \infty$, and  $\alpha_n/\sqrt{n} \rightarrow c_{3} \in (0, \infty)$.
Then $\tilde{\b}_n$ satisfies the following properties:
\begin{enumerate}
\item[\emph{(1)}]  Consistency in variable selection: \[ \lim_{n \rightarrow \infty} P( \tilde{{\cal A}}_n={\cal A})=1. \]
\item[\emph{(2)}]  Asymptotic normality: \[ \sqrt{n}(\tilde{\b}_{n1} - \b^{*}_{1}) \overset{d}{\longrightarrow} N(\0, \sigma^2 \C_{11}^{-1}). \]
\end{enumerate}
\end{theorem}%

%It is worth pointing out that the condition in Theorem~\ref{thm:oracle2} is weaker than that in Theorem~\ref{thm:oracle1}, because
%$c_3\in (0, \infty]$ but $c_1\in (0, \infty)$.   This makes sense because Theorem~\ref{thm:oracle1} studies a one-step approximation
%of the case in Theorem~\ref{thm:oracle2}.
It is worth noting that we set $\eta_n= \frac{\lambda_n \alpha_n} {\sqrt{2 \alpha_n{+}1}-1}$ to define $\Phi(|b|; \alpha)$
in (\ref{eqn:rho}). In this setting, if we assume that $\lambda_{n}/n^{1/4} \to \infty$, $\lambda_{n}/n^{1/2} \to 0$, and
$\alpha/n^{1/2} \to c_3 \in (0, \infty)$, then we can obtain that $\eta_{n}/n^{1/2} \to \infty$, $\lambda_{n}/n^{3/4} \to 0$, and
$\alpha/n^{1/2} \to c_3 \in (0, \infty)$; that is, the conditions in Theorem~\ref{thm:oracle2} meet.
Consider that the condition  $\alpha_n/\sqrt{n} \rightarrow c_3$ implies that $\alpha_n \rightarrow \infty$. Hence, $\frac{\sqrt{1{+} 2\alpha_n |b_j|} {-}1}{\sqrt{2 \alpha_n{+}1}-1} \rightarrow |b_j|^{1/2}$. Thus, it follows from Theorem~\ref{thm:oracle2} that the $\ell_{1/2}$ penalty can also result in an estimator with the oracle property under the conditions $\lambda_{n}/n^{1/4} \to \infty$ and $\lambda_{n}/n^{1/2} \to 0$.

On the other hand, $\liml_{\alpha_n \to 0+} \frac{\sqrt{1{+} 2\alpha_n |b_j|} {-}1}{\sqrt{2 \alpha_n{+}1}-1} = \liml_{\alpha_n \to 0+} \frac{\sqrt{1{+} 2\alpha_n |b_j|} {-}1}{\alpha_n} = |b_j|$. Thus, it is of great interest to explore the asymptotic property  of the sparse estimator when $\alpha_n \to 0$.
In particular, we have the following theorem.

\begin{theorem}  \label{thm:asymptotic}
Assume
$\liml_{n \to \infty} \alpha_n =0$. If $\liml_{n \to \infty} \frac{\eta_n}{\sqrt{n}} = 2 c_3 \in [0, \infty)$,
then  $\tilde{\b}_{n} \overset{p}{\longrightarrow} \b^{*}$. Furthermore, if $\liml_{n \to \infty} \frac{\eta_n}{\sqrt{n}} =0$,
then $\sqrt{n}(\tilde{\b}_{n} {-} \b^{*})
\overset{d}{\longrightarrow} N(\0, \sigma^2 \C^{-1})$.
%\begin{enumerate}
%\item[\emph{(1)}]  If $c_3 =0$, then $\sqrt{n}(\tilde{\b}_{n} {-} \b^{*})
%\overset{d}{\longrightarrow} N(\0, \sigma^2 \C^{-1})$;
%\item[\emph{(2)}] If $c_3 \in (0, \infty]$, then $\tilde{\b}_{n} \overset{p}{\longrightarrow} \b^{*}$.
%\end{enumerate}
\end{theorem}%

Recall that the conditions $\eta_n/\sqrt{n} \rightarrow \infty$ and  $\alpha_n/\sqrt{n} \rightarrow c_{3} \in (0, \infty)$ in Theorem~\ref{thm:oracle2} imply that $\eta_n \rightarrow \infty$ and $\alpha_n \rightarrow \infty$.
Consequently, $\eta_n> \frac{1}{\alpha_n}$ when $n$ is sufficiently large. It then follows from Theorem~\ref{thm:sparsty} that the thresholding
rule is discontinuous under these conditions. This leads us to an interesting phenomenon; that is, the
``oracle properties" do not always accompany ``continuity." Our following empirical analysis shows that ``continuity" is indeed very necessary
for the coordinate descent algorithm to achieve good performance.
However, the conditions for $\alpha_n$ and $\eta_n$ in Theorem~\ref{thm:asymptotic} are always able to hold $\eta_n< \frac{1}{\alpha_n}$.
For example, we take $\alpha_n=n^{-\frac{\epsilon_1+1}{2}}$ and $\eta_n=n^{\frac{1+\epsilon_2}{2}}$ for any $0\leq \epsilon_2 < \epsilon_1$, which
holds $\eta_n \alpha_n<1$ true.

In the previous discussion, $p$ is fixed. We are also interested in the asymptotic properties when $r$   and $p$ rely on $n$. That is, $r=r_n$ and $p=p_n$ are allowed to grow as $n$ increases.
In this case, we are concerned with  notion of sign consistency of the estimate $\tilde{\b}_n$ with the true  $\b^{*}$. In particular,
it is said that $\tilde{\b}_n$ is equal  to $\b^{*}$ in sign, which is written as $\tilde{\b}_n \overset{s}{=} \b^{*}$, if and only if
$\sgn(\tilde{\b}_n)= \sgn(\b^{*})$.

In order to address sign consistency, we consider a so-called  strong irrepresentable condition~\citep{ZhaoYu:2006}.
Assume that $\C_{11}$ is invertible. The strong  irrepresentable condition is that there a positive constant
number $\gamma$ such that
\[
\|\C_{21}\C_{11}^{-1} \sgn(\b^{*}_{1})\|_{\infty} \leq 1- \gamma.
\]
%\begin{definition} There exists a positive constant  vector $\gamb$
%\end{definition}
Following the setting of \citet{ZhaoYu:2006},
we further make the following assumptions  on
$\C_n$, $r_n$ and $\b^{*}$. Specifically, there exist $M_1>0$,  $M_2>0$ and $M_3>0$ as well as $0\leq c_4 < c_5\leq 1$ such that
\begin{eqnarray}
\frac{1}{n} \x_i^T \x_i &\leq & M_1 \quad \mbox{ for  } i=1, \ldots, n, \label{ass:3} \\
\a^T \sum_{i=1}^r \x_i \x_i^T \a &\geq &  M_2  \quad \mbox{ for any } \|\a\|_2=1. \label{ass:4} \\
r_n &=& O(n^{c_4}),  \label{ass:5} \\
n^{\frac{1-c_5}{2}} \min_{j=1, \cdots, r} |b_{j}^{*}| &\geq & M_3. \label{ass:6}
\end{eqnarray}

The detailed interpretation for these conditions can be also found in \citet{ZhaoYu:2006}. Roughly speaking, Conditions~\ref{ass:3}
and \ref{ass:4} are alternative as the previous assumption on $\frac{1}{n}\X^T \X$ when $p_n$ and $r_n$ are fixed.
Condition~\ref{ass:5} implies that $ r_n/n \to 0$, while Condition~\ref{ass:6} shows that there exists a gap of size $n^{c_5}$ between
the decay rate of $\b_1^{*}$ and $n^{-\frac{1}{2}}$.

\begin{theorem}  \label{thm:asymptotic2}
Assume $\eta_n<1/\alpha_n$,
$\liml_{n \to \infty} \alpha_n =0$ and  Conditions~(\ref{ass:3})-(\ref{ass:6}) are satisfied. Under the strong irrepresentable condition,
if $p_n = O(e^{n^{c_6}})$ and $\eta_n \varpropto n^{\frac{1+c_7}{2}}$ where $c_6< c_7 < c_5- c_4$,
then
\[
\Pr(\tilde{\b}_n \overset{s}{=} \b^{*})\geq 1- o(e^{-n^{c_6}}) \to 1 \quad \mbox{ as } \quad n \to \infty.
\]
\end{theorem}%

This theorem is similar to Theorem~4 of \citet{ZhaoYu:2006}. Consider that $\tilde{\b}_n$
is the solution of the problem in (\ref{eqn:sqre}). Thus,
\[
0 \in  (\X \tilde{\b}_n {-} \y)^T \x_{\cdot j} + \frac{\eta_n}{\sqrt{1{+} 2\alpha_n |\tilde{b}_{nj}|}}  \partial |\tilde{b}_{nj}|, \quad j=1, \ldots, p.  \]
Under the condition $\alpha_n\to 0$, we have
\[
0 \in \lim_{n \to \infty} \Big\{ (\X \tilde{\b}_n {-} \y)^T \x_{\cdot j} + \frac{\eta_n}{\sqrt{1 {+} 2\alpha_n |\tilde{b}_{nj}|}}  \partial |\tilde{b}_{nj}| \Big\}
= \lim_{n \to \infty} \Big\{ (\X \tilde{\b}_n {-} \y)^T \x_{\cdot j} + {\eta_n}  \partial |\tilde{b}_{nj}| \Big\}
\]
for $ j=1, \ldots, p$. Since the minimizer of the conventional lasso exists and unique (denote  $\hat{\b}_{0}$), the above
relationship implies that $\liml_{n \to \infty} \tilde{\b}_n=\liml_{n \to \infty} \hat{\b}_{0}$. Accordingly,
based on Theorem~4 of \citet{ZhaoYu:2006}, we obtain the result in Theorem~\ref{thm:asymptotic2}.

%Recently, \citet{ZhangZhang2012} presented a general theory of nonconvex regularization for sparse learning problems.
%Their work is built on the following four conditions on the penalty function $P(b; \lambda)$: (i) $P(0; \lambda)=0$; (ii)  $P(-b; \lambda)=P(b; \lambda)$;
%(iii) $P(b; \lambda)$ is increasing in $b$ on $[0, \infty)$; (iv) $P(b; \lambda)$ is subadditive w.r.t.\
%$b\geq 0$, i.e., $P(u+v; \lambda) \leq P(u; \lambda)+ P(v; \lambda)$ for any $u\geq 0$ and $v\geq 0$.
%It is easily seen  that the KEP function $\Psi(|b|; \eta, \alpha)$ as a function of $b$ satisfies these four conditions.
%Thus, we can directly apply the theoretical analysis of \citet{ZhangZhang2012}  to our special case.
%Here we omit the details.

\section{Experimental Analysis} \label{sec:experiment}

In  Section~\ref{sec:simul} we  conduct a simulation analysis of KEP in sparsity modeling. This analysis is based on
Theorem~\ref{thm:asymptotic}. In Sections~\ref{sec:simreg} and \ref{sec:reg_real} we evaluate
the  performance of the KEP-based coordinate descent algorithm given in Algorithm~\ref{alg:coord} in linear regression problems on  simulated data and real data, respectively.
We also conduct comparisons with the coordinate descent algorithms based on the $\ell_{1}$-norm,
$\ell_{1/2}$-norm and MCP, respectively.

\subsection{Simulation Analysis}
\label{sec:simul}

%for i=1:5
%    b(i) = 0.2*i;
%    b(p/2+i) = b(i);
%end

In this simulation analysis, we use a data model same to that in \cite{MazumderSparsenet:11}.  In particular, we  generate  data
from the following model:
\begin{equation*}
 y = \x^T \b + \sigma e
\end{equation*}
where $e \sim N(0, 1)$, and $\b$ is a $p$-dimensional vector with
only 10 nonzero elements: $b_i =  b_{i+ {p}/{2}} = 0.2 i, \; i=1, \ldots,
5$.
Each data point $\x$
is sampled from a multivariate normal distribution with zero mean and
covariance matrix $\Sigma = \{0.7^{|i-j|}\}_{1 \leq i, j \leq p}$.
We choose $\sigma$ such that the Signal-to-Noise Ratio (SNR), which is
\begin{equation*}
 \mathrm{SNR} = \frac{\sqrt{\b^T \Sigma \b}}{\sigma},
\end{equation*}
is a specified value.
Let $\hat{\b}$
denote the solution obtained from each algorithm.  We use a standardized
prediction error (SPE) and a feature selection error (FSE)  as measure metrics. The SPE  is defined as
\begin{equation*}
 \textrm{SPE} = \frac{\sum_{i=1}^{m} (y_i - \x_i^T \hat{\b})^2}{m \sigma^2}
\end{equation*}
and the FSE is proportion of coefficients in
$\hat{\b}$ which is wrongly set to zero or nonzero based on the true $\b$.

%\subsubsection{Simulation Analysis with Prespecified $\eta_n$ and $\alpha_n$}

In this simulation we prespecify the values of hyperparameters $\eta_n$ and $\alpha_n$.
Based on Theorem~\ref{thm:asymptotic}, we particularly set $\eta_n=n^{1/4}$ and $\alpha_n = n^{-1/2}$. Clearly,
in this setting we always have $\eta_n \alpha_n <1$.
We also implement the MCP-based coordinate descent method with the same setting, and the lasso-based coordinate descent method with
$\lambda_n=\eta_n=n^{1/4}$. Our simulation analysis is performed on the training datasets with different sizes ($n$) and a fixed  $p$ (that is, $p=200$). But all the corresponding test datasets  include $m=1000$ samples.

We use  different settings of  $n$ and {SNR} to generate the training datasets.
Tables~\ref{tab:sim_01}-\ref{tab:sim_04} report the results over 20 repeats for each setting.
We can see that when $n$ takes a smaller value, the performance of the KEP penalty is significantly better than that of  MEP
and of the $\ell_1$-norm. As $n$ takes a larger value, the performances of all the three penalties become better. Especially,
the KEP and MCP are both competitive. Moreover, the three penalties can almost fully capture the model sparsity
for a large $n$. Thus, in this case, $\eta_n=n^{1/4}$ and $\alpha_n = n^{-1/2}$ are  good choices for KEP and MCP.

Additionally, for a larger SNR, the performances of the MCP and $\ell_1$-norm
become worse. In contrast, the KEP penalty still works well. This shows that  KEP
is more robust  than the MCP and $\ell_1$-norm.
Finally, Figure~\ref{fig:sim00}
depicts the convergence procedure of the coordinate descent algorithm. As we see,
the algorithms with the KEP, MCP and $\ell_{1}$-norm are efficient, because they get convergence after about 10 steps.

%SNR = 3									
%N	100	200	400	800	1600	3200	6400	12800	
%L_half_4	1.756958323	1.138794543	1.063545658	1.066754537	1.077584476	1.02623292	1.026820561	1.019789966	
%KEP_4	2.249577068	1.353418917	1.072573975	1.08967762	1.080853263	1.045295882	1.04642664	1.022688775	
%Lasso_5	2.826031056	1.788626378	1.527591944	1.54888931	1.389433314	1.365320858	1.330634877	1.266912766	
%MCP_5	2.309856531	1.397226976	1.195996137	1.288651057	1.125991161	1.051634489	1.046015805	1.03045927	
%KEP_5	1.978868918	1.2101729	1.09784921	1.098474987	1.06548844	1.046844904	1.045629066	1.030893263	

\begin{table}[!ht] \setlength{\tabcolsep}{1.7pt}
\caption{Simulation results on datasets with $p=200$ and $SNR=3.0$ under the setting $\eta_n=n^{1/4}$ and $\alpha_n = n^{-1/2}$ for MCP and KEP, and $\lambda_n=n^{1/4}$ for Lasso.}
\label{tab:sim_01}
\begin{center}
\begin{tabular}{l|lr|lr|lr|lr|lr|lr} \hline
        & \multicolumn{2}{|c|}{n=100}  & \multicolumn{2}{|c|}{n=200} & \multicolumn{2}{|c|}{n=400} & \multicolumn{2}{|c|}{n=1600} & \multicolumn{2}{|c|}{n=6400} & \multicolumn{2}{|c}{n=12800} \\
        & SPE & ``FSE''  & SPE & ``FSE'' & SPE & ``FSE'' & SPE & ``FSE'' & SPE & ``FSE'' & SPE & ``FSE''   \\
\hline
 KEP & \textbf{1.979} &	0.020 & \textbf{1.210} & 0.010 & \textbf{1.098} & 0.010 & \textbf{1.065} &	0.000 & \textbf{1.045} & 0.000 & 1.031  & 0.000   \\
 MCP     & 2.310 &	0.040     & 1.397  &	0.020   & 1.196 & 0.010 & 1.126 &	0.005  & 1.046 & 0.000	& \textbf{1.030}  & 0.000  \\
 Lasso  & 2.826 & 0.020      & 1.789  & 0.010     & 1.528 & 0.010  & 1.389 &	0.005  & 1.331 & 0.000 &  1.267 & 0.000 \\
\hline
\end{tabular}
\end{center}
\end{table}

%									
%SNR = 6									
%N	100	200	400	800	1600	3200	6400	12800	
%Lasso_5	6.042015841	4.863494633	3.013342053	3.297439078	2.505695959	2.399484431	2.231194308	1.948002338	
%L_half_4	1.911823752	1.315475449	1.183979193	1.165374752	1.154340774	1.108868485	1.029444207	1.022118385	
%MCP_5	3.7791786	2.594354951	2.415029794	1.243038072	1.21289886	1.142470945	1.084214435	1.062318462	
%KEP_4	2.511234301	2.275796735	1.225793276	1.195882853	1.164821074	1.159487817	1.15273516	1.037161036	
%KEP_5	2.101680124	1.679569171	1.310774834	1.311789446	1.191409132	1.129514513	1.090541266	1.060771023	
		
\begin{table}[!ht] \setlength{\tabcolsep}{1.7pt}
\caption{Simulation results on datasets with $p=200$ and $SNR=6.0$ under the setting $\eta_n=n^{1/4}$ and $\alpha_n = n^{-1/2}$ for MCP and KEP, and $\lambda_n=n^{1/4}$ for Lasso.}
\label{tab:sim_02}
\begin{center}
\begin{tabular}{l|lr|lr|lr|lr|lr|lr} \hline
        & \multicolumn{2}{|c|}{n=100}  & \multicolumn{2}{|c|}{n=200} & \multicolumn{2}{|c|}{n=400} & \multicolumn{2}{|c|}{n=1600} & \multicolumn{2}{|c}{n=6400} & \multicolumn{2}{|c}{n=12800} \\
        & SPE & ``FSE''  & SPE & ``FSE'' & SPE & ``FSE'' & SPE & ``FSE'' & SPE & ``FSE'' & SPE & ``FSE'' \\
\hline
 KEP & \textbf{2.102} &	0.010 & \textbf{1.679} & 0.000 & \textbf{1.311} &	0.005 & {1.191} & 0.000 & 1.090 & 0.000 & \textbf{1.061}  & 0.000 \\
 MCP     & 3.779 &	0.020     & 2.594  &	0.020      & 2.415 &	0.005  & 1.213	& 0.010 & \textbf{1.084} & 0.000 & {1.062} & 0.000    \\
 Lasso  & 6.042 & 0.010      & 4.863  & 0.010        & 3.013 &	0.010  & 2.506 &   0.010 & 2.231 &	0.005 & 1.948 &	0.000 \\
\hline
\end{tabular}
\end{center}
\end{table}

\begin{table}[!ht] \setlength{\tabcolsep}{1.7pt}
\caption{Simulation results on datasets with $p=200$ and $SNR=9.0$ under the setting $\eta_n=n^{1/4}$ and $\alpha_n = n^{-1/2}$ for MCP and KEP, and $\lambda_n=n^{1/4}$ for Lasso.}
\label{tab:sim_03}
\begin{center}
\begin{tabular}{l|lr|lr|lr|lr|lr|lr} \hline
        & \multicolumn{2}{|c|}{n=100}  & \multicolumn{2}{|c|}{n=200} & \multicolumn{2}{|c|}{n=400} & \multicolumn{2}{|c|}{n=1600} & \multicolumn{2}{|c}{n=6400} & \multicolumn{2}{|c}{n=12800} \\
        & SPE & ``FSE''  & SPE & ``FSE'' & SPE & ``FSE'' & SPE & ``FSE'' & SPE & ``FSE'' & SPE & ``FSE'' \\
\hline
 KEP & \textbf{2.892} &	0.010 & \textbf{2.379} & 0.020 & \textbf{1.675} &	0.005 & \textbf{1.327} & 0.000 &{1.201} & 0.000 & {1.119}  & 0.000 \\
 MCP     & 6.564 &	0.025     & 5.123  &	0.030      & 2.669 &	0.001  & 1.405	& 0.005 &  \textbf{1.186} & 0.000 & \textbf{1.116} & 0.000    \\
 Lasso  & 9.472 & 0.050      & 8.479  & 0.010        & 6.404 &	0.005  & 4.444 &  0.010  & 3.760 &	 0.010 & 3.143 &	0.000 \\
\hline
\end{tabular}
\end{center}
\end{table}
							
%SNR = 9									
%N	100	200	400	800	1600	3200	6400	12800	
%Lasso_5	9.471947473	8.479311853	6.4042225	5.266627945	4.444407178	3.998399393	3.760046498	3.142567581	
%L_half_4	3.525718754	1.649282151	1.340064815	1.357233522	1.326280224	1.066950147	1.051286642	1.024576383	
%MCP_5	6.563625064	5.122586227	2.668686339	1.933058956	1.405452072	1.212514478	1.185883688	1.115587517	
%KEP_4	3.708106736	2.012124808	1.419627304	1.416901825	1.35499037	1.329088552	1.333235317	1.052098266	
%KEP_5	2.892241173	2.378746359	1.675189555	1.508554847	1.327252321	1.219594867	1.201391651	1.119532325	
%									

\begin{table}[!ht] \setlength{\tabcolsep}{1.7pt}
\caption{Simulation results on datasets with $p=200$ and $SNR=12.0$ under the setting $\eta_n=n^{1/4}$ and $\alpha_n = n^{-1/2}$ for MCP and KEP, and $\lambda_n=n^{1/4}$ for Lasso.}
\label{tab:sim_04}
\begin{center}
\begin{tabular}{l|lr|lr|lr|lr|lr|lr} \hline
        & \multicolumn{2}{|c|}{n=100}  & \multicolumn{2}{|c|}{n=200} & \multicolumn{2}{|c|}{n=400} & \multicolumn{2}{|c|}{n=1600} & \multicolumn{2}{|c}{n=6400} & \multicolumn{2}{|c}{n=12800} \\
        & SPE & ``FSE''  & SPE & ``FSE'' & SPE & ``FSE'' & SPE & ``FSE'' & SPE & ``FSE'' & SPE & ``FSE'' \\
\hline
 KEP & \textbf{3.972} &	0.010 & \textbf{3.381} & 0.010 & \textbf{2.138} &	0.005 & \textbf{1.748} & 0.000 &{1.334} & 0.000 & {1.233}  & 0.000 \\
 MCP     & 5.616 &	0.020     & 7.309  &	0.020      & 6.341 &	0.020  & 1.898	& 0.005 &  \textbf{1.303} & 0.000 & \textbf{1.206} & 0.000    \\
 Lasso  & 15.904 & 0.010      & 13.938  & 0.010        & 8.235 &	0.010  & 7.739 &  0.010  & 5.788 &	0.005 & 4.961 &	0.000 \\
\hline
\end{tabular}
\end{center}
\end{table}

%SNR = 12									
%N	100	200	400	800	1600	3200	6400	12800	
%Lasso_5	15.90359127	13.93829846	8.235200592	9.522325098	7.7386157	6.848453787	5.788147003	4.96085846	
%L_half_4	3.932614259	2.168206018	1.670158528	1.675714802	1.617752305	1.321422594	1.078680139	1.043469958	
%MCP_5	5.616223925	7.309021548	6.340776917	3.019972984	1.89838686	1.430523289	1.303314651	1.205520112	
%KEP_4	6.546973245	4.246062742	1.806349807	1.808568155	1.697313529	1.585194479	1.558754233	1.098519624	
%KEP_5	3.972408794	3.380893638	2.138243311	2.031079501	1.747572347	1.473784328	1.334373402	1.232776202	

\begin{figure}[!ht]
\centering
%% \begin{tabular}{ccc}
%%\hspace{-1.5cm}
\subfigure[$\ell_1$-norm]{\includegraphics[width=50mm,height=40mm]{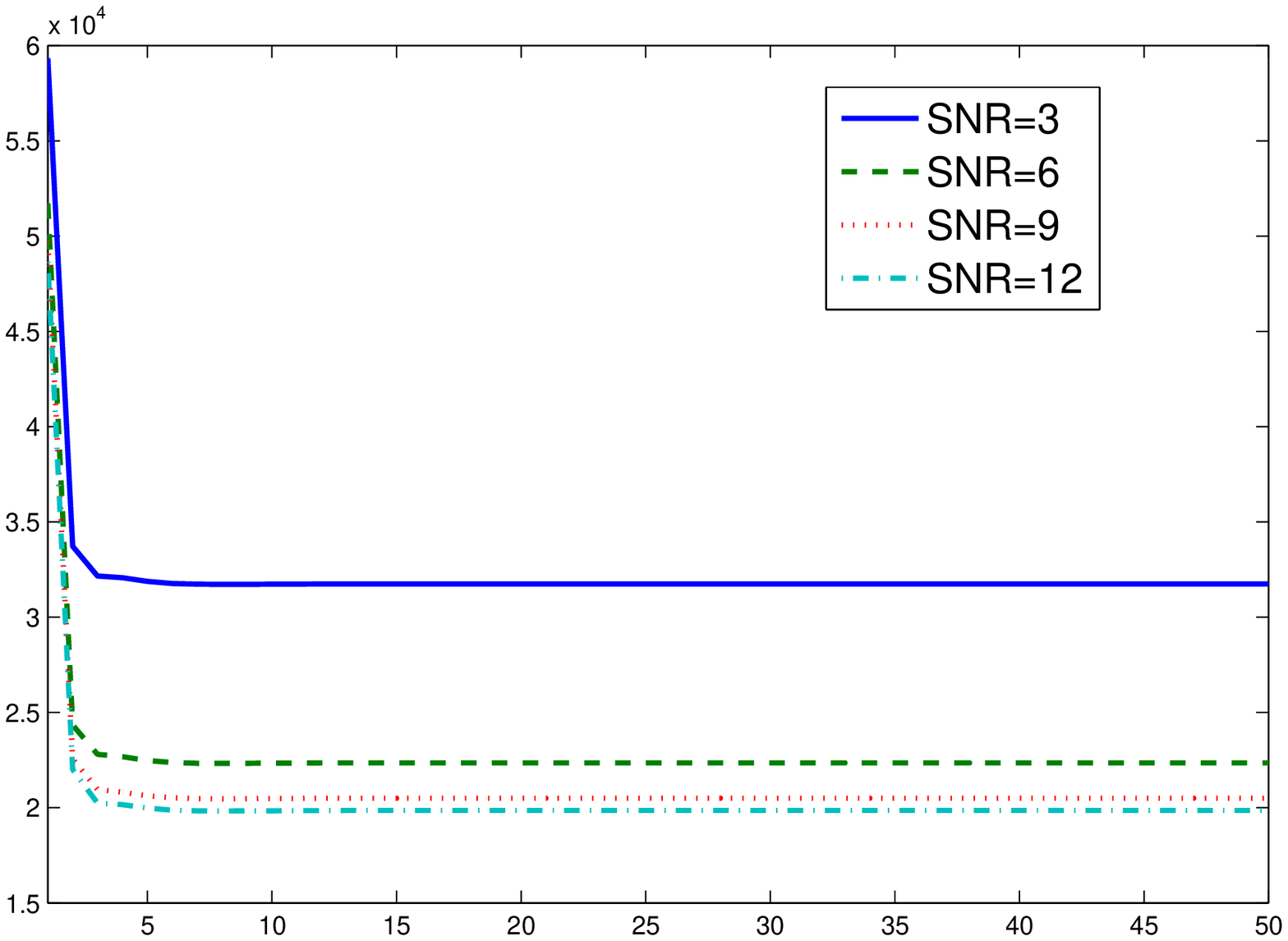}}
\subfigure[KEP]{\includegraphics[width=50mm,height=40mm]{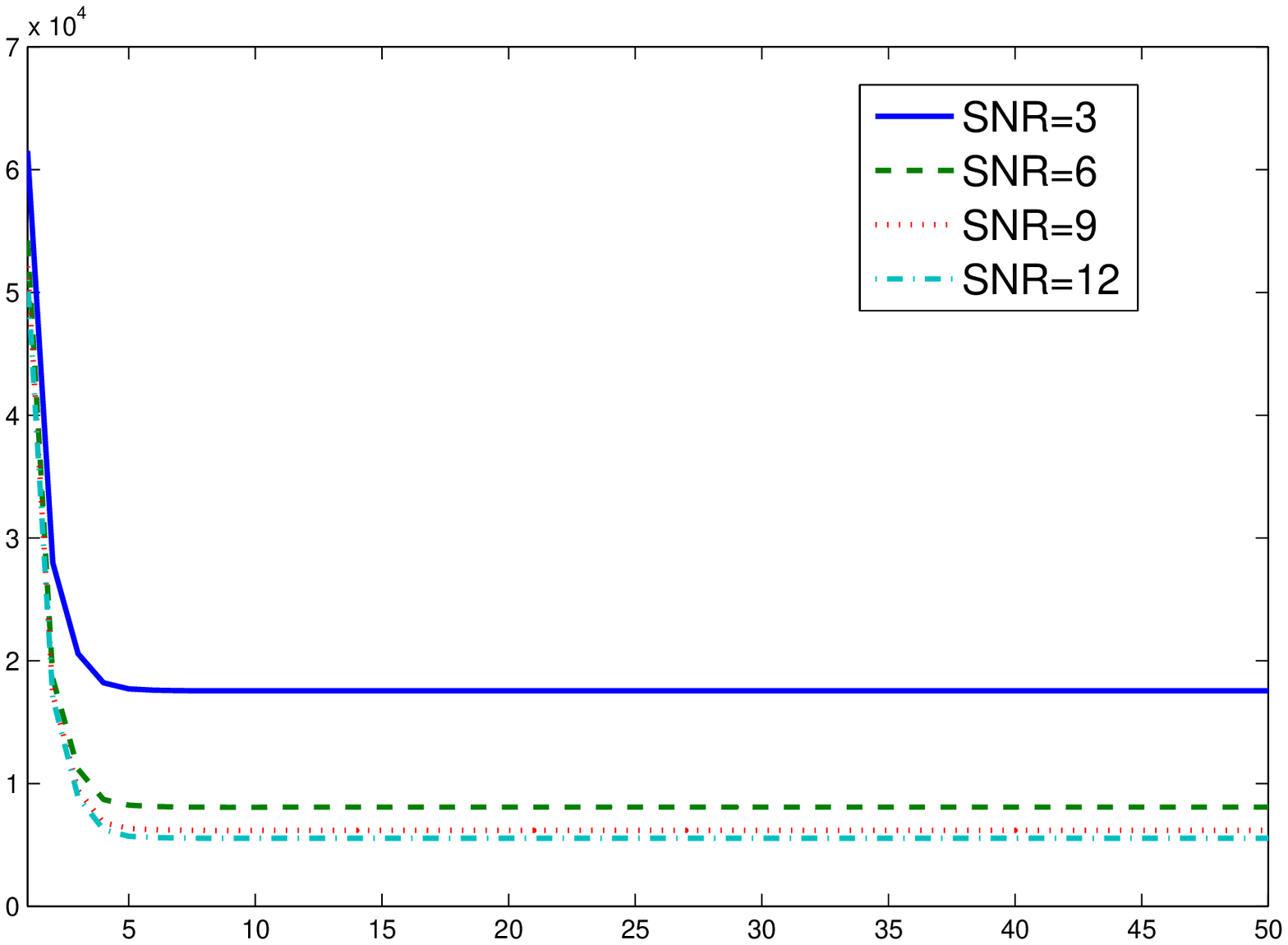}} \subfigure[MCP]{\includegraphics[width=50mm,height=40mm]{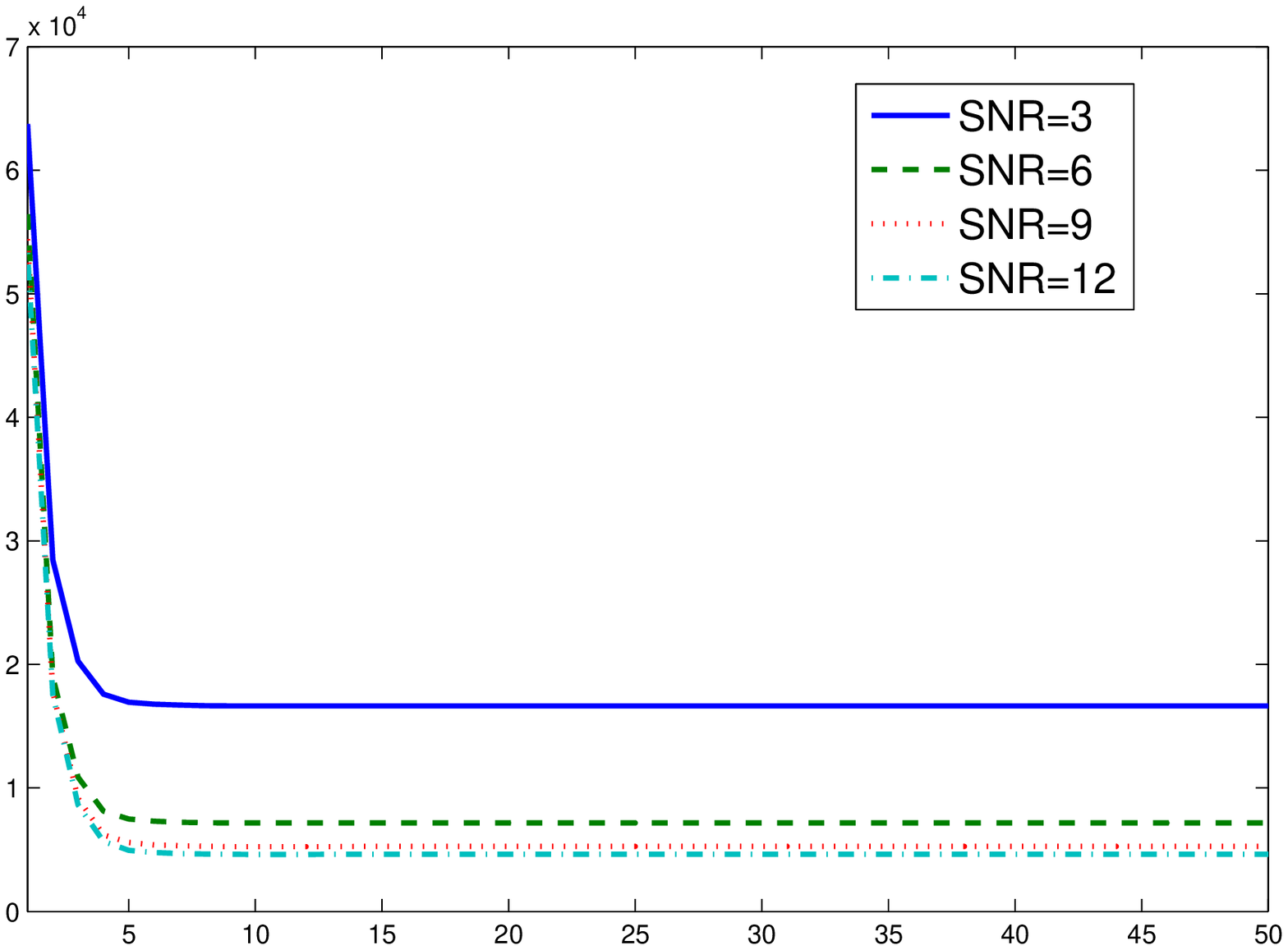}}\\
%\end{tabular}
\caption{Convergence procedures  for coordinate descent iterations with the $\ell_1$-norm,  KEP and MCP over the datasets with $p=200$ and $n=12800$. Here $\eta_n=n^{1/4}$ and $\alpha_n = n^{-1/2}$ for KEP and MCP, and $\lambda_n = n^{1/4}$ for Lasso.}
\label{fig:sim00}
\end{figure}

\subsection{Linear Regression on Simulated data}
\label{sec:simreg}

In this paper our principal focus has been to provide
KEP  with which  the
iteratively reweighted
$\ell_q$ method of \cite{Daubechies:2010} can be derived. We particularly
study the case that $q=1$, because the corresponding KEP  is
nonconvex and has strong ability
in sparsity modeling.
Moreover, we have also proposed the coordinate descent (CD) algorithm  based on KEP in
Section~\ref{sec:cda}. Thus, it is
interesting to conduct empirical comparison of  our CD with  the iterative reweighted (IR) method (see Section~\ref{sec:cda}).
For description simplicity, we denote them  by \emph{KEP-CD} and
\emph{KEP-IR}.
%In Section~\ref{sec:mixture} we have  presented the MAP estimate based on the EM algorithm.
%We have shown that the EM algorithm is essentially equivalent to the the iterative reweighted  method.
%Thus, we do not include experimental results with the EM algorithm here.

We now conduct comparisons of the methods based on KEP  with the Lasso,
the adaptive Lasso (AdaLasso)~\citep{ZouJasa:2006}, the method based on the $\ell_{1/2}$-norm penalty,
and the SparseNet based on the MCP
penalty~\citep{MazumderSparsenet:11}.
All these methods are  solved by using the  coordinate descent
algorithm.  Moreover, the hyperparameters ($\eta$, $\alpha$ or $\lambda$) involved in all the methods are selected via cross validation.
%is also used to select the tuning parameters.  In  KEP-IR,  $\lambda$ and $k$ are selected via cross validation,
%and cross validation is used to select $k$
%during the update of $\epsilon$.
The experiments are also implemented over the previous simulation data model. In particular,
we generate 12 datasets based on $ \mathrm{SNR} =3.0, 6.0, 9.0, 12$
and $p=50, 200, 500$ to implement the simulation.
Our experimental analysis is performed on the training datasets of   $n=100$ samples  and the corresponding test datasets
of $m=1000$ samples.

Tables~\ref{tab:sim_res1}-\ref{tab:sim_res3} report the average results over 20 repeats. From
them, we can see that the KEP penalty function is competitive
with the MCP penalty, $\ell_{1/2}$-norm  and $\ell_1$-norm in both prediction accuracy and feature selection accuracy.
In most cases,  KEP can lead to more accurate prediction results than the rest three penalty functions.
Additionally, we see that nonconvex penalization  outperforms  convex penalization in sparsity, and
almost outperforms  convex penalization in regression accuracy. Thus, nonconvex penalization
is an effective approach for high-dimensional data modeling.
%The reason is  that  sparsity  is not always in company with prediction accuracy.

Comparing  KEP-CD and KEP-IR,
we can see that their performances are competitive. However, KEP-CD is computationally more efficient than
KEP-IR.  KEP-CD usually takes about 10
iterations to get convergence (see Figure~\ref{fig:convergence} for illustration).
As discussed in Section~\ref{sec:cda}, KEP-IR uses  two nested loops to iterate over all the elements of $\b$. In the inner loop, KEP-IR uses a LLA of the original problem in each coordinate, which needs to take several iterations to get convergence.
In contrast, KEP-CD just takes one step to obtain the exact solution of the original problem in each coordinate.
%KEP-IR is sensitive on the value of $k$, which is selected via cross validation.

From the experimental results, we see that the $\ell_{1/2}$-norm and MCP are slightly stronger than
KEP in sparsity ability. This makes sense because KEP with $q=1$ bridges the $\ell_1$-norm and the $\ell_{1/2}$
norm, and we have $M(|b|)\leq K(|b|)\leq |b|$ (see Section~\ref{sec:related}).
However,  the $\ell_{1/2}$-norm indeed suffers from the numerical
instable problem. It is seen from Tables~\ref{tab:sim_res1}-\ref{tab:sim_res3} that relative to the other methods, the prediction performance with the $\ell_{1/2}$-norm becomes worse as $p$ grows.
As for the MCP-based method, \cite{MazumderSparsenet:11} showed that a recalibration strategy can improve performance.

\begin{table}[!ht]
\caption{Simulation results on dataset with $p=50$ and $n=100$}
\label{tab:sim_res1}
\begin{center}
\begin{tabular}{l|ll|ll|ll|ll} \hline
        & \multicolumn{2}{c}{SNR=3.0}  & \multicolumn{2}{c}{SNR=6.0} & \multicolumn{2}{c}{SNR=9.0}  & \multicolumn{2}{c}{SNR=12.0} \\
        & SPE & ``FSE''  & SPE & ``FSE'' & SPE & ``FSE''  & SPE & ``FSE'' \\
\hline \\
 KEP-CD & \textbf{1.245}  &	0.071 & \textbf{1.215} & 0.069 & \textbf{1.169} &	0.048 & \textbf{1.147} & 0.014 \\
KEP-IR  & 1.264  &	0.088 & 1.243 & 0.056 & 1.199 & 0.030 & 1.148  & 0.032 \\
 %KEP-ALfp & 1.23          & \textbf{0.015} & 1.15          & 0.002 \\
 %KEP-ALap   & 1.28          & 0.072          & 1.30          & 0.067     \\
       \\
 MCP    & 1.269 &	0.040 & 1.237 &	0.053 & 1.203 &	0.028 &	1.169 & 0.025    \\
 $\ell_{1/2}$ & 1.276 &	0.085 &	1.255	& 0.047 & 1.196 &	0.032 &	1.252 &	0.012 \\
 AdaLasso & 1.275 &	0.096 &  1.291 & 0.123 & 1.215 & 0.058 & 1.175 &	0.030   \\
 Lasso  & 1.361 & 0.166  & 1.337 & 0.160 &	1.253 &	0.130 &	1.220 &	0.139   \\
\hline
\end{tabular}
\end{center}
\end{table}

%SNR	6		9		12		3	
%	SPE	FSE	SPE	FSE	SPE	FSE	SPE	FSE
%RP-CD	1.215 	0.069 	1.169 	0.048 	1.147 	0.014 	1.245 	0.071
%RP-IR	1.243 	0.056 	1.199 	0.030 	1.148 	0.032 	1.264 	0.088
%								
%l_half	1.255 	0.047 	1.196 	0.032 	1.152 	0.012 	1.276 	0.085
%MCP	1.237 	0.053 	1.203 	0.028 	1.169 	0.025 	1.269 	0.040
%AdaLasso	1.291 	0.123 	1.215 	0.058 	1.175 	0.030 	1.275 	0.096
%Lasso	1.337 	0.160 	1.253 	0.130 	1.220 	0.139 	1.361 	0.166

\begin{table}[!ht]
\caption{Simulation results on dataset with $p=200$ and $n=100$}
\label{tab:sim_res2}
\begin{center}
\begin{tabular}{l|ll|ll|ll|ll} \hline
        & \multicolumn{2}{c}{SNR=3.0}   & \multicolumn{2}{c}{SNR=6.0} & \multicolumn{2}{c}{SNR=9.0}  & \multicolumn{2}{c}{SNR=12.0} \\
         & SPE & ``FSE''  & SPE & ``FSE'' & SPE & ``FSE''  & SPE & ``FSE'' \\
\hline \\
 KEP-CD & 1.248 & 0.038 & \textbf{1.224} & {0.024} & \textbf{1.197} & {0.018} & {\bf 1.179} &	0.009 \\
 KEP-IR & \textbf{1.240} & 0.035 & 1.225 & 0.024 & 1.203  & 0.009 & 1.181 & 0.007 \\
       \\
 MCP    & 1.246 &	0.020 & 1.235 &	0.040 &	1.219 &	0.015 &	1.196 &	0.015    \\
$\ell_{1/2}$-CD & 1.296 &	0.021 & 1.253 & 0.015 &	1.233 &	0.016 & 1.215 &	0.011 \\
 AdaLasso & 1.347 &	0.041 &  1.274 & 0.035 & 1.261 & 0.020 & 1.203 & 0.013    \\
 Lasso  & 1.356  & 0.078 & 1.368 & 0.069 & 1.280 & 0.072 & 1.300 & 0.063    \\
\hline
\end{tabular}
\end{center}
\end{table}

%	SPE	FSE	SPE	FSE	SPE	FSE	SPE	FSE
%RP-CD	1.224 	0.024 	1.197 	0.018 	1.179 	0.009 	1.258 	0.038
%RP-IR	1.225 	0.015 	1.203 	0.009 	1.181 	0.007 	1.240 	0.035
%								
%l_half	1.253 	0.015 	1.233 	0.016 	1.215 	0.011 	1.296 	0.021
%MCP	1.235 	0.040 	1.219 	0.015 	1.196 	0.015 	1.246 	0.020
%AdaLasso	1.274 	0.035 	1.261 	0.020 	1.203 	0.001 	1.347 	0.041
%Lasso	1.368 	0.069 	1.280 	0.072 	1.300 	0.063 	1.356 	0.078

\begin{table}[!ht]
\caption{Simulation results on dataset with $p=500$ and $n=100$}
\label{tab:sim_res3}
\begin{center}
\begin{tabular}{l|ll|ll|ll|ll} \hline
         & \multicolumn{2}{c}{SNR=3.0}  & \multicolumn{2}{c}{SNR=6.0} & \multicolumn{2}{c}{SNR=9.0}  & \multicolumn{2}{c}{SNR=12.0} \\
          & SPE & ``FSE''  & SPE & ``FSE'' & SPE & ``FSE''  & SPE & ``FSE'' \\
\hline \\
 KEP-CD & 1.327 & 0.023 & {\bf 1.273} & 0.013 & {1.247} &	0.002 &	{1.215} &	0.009 \\
KEP-IR  & \textbf{1.319} &	0.014 & {1.292}   & 0.008 & \textbf{1.242} & 0.003 & 1.225 & 0.009  \\
       \\
 MCP    & 1.338 &	0.016  & 1.284 & 0.012 &	1.260  &	0.014 &	\textbf{1.195} &	0.010    \\
$\ell_{1/2}$-CD & 1.383 &	0.051 & 1.360 & 0.003 &	1.272  &	0.002 &	1.251 &	0.003 \\
 AdaLasso & 1.360 &	0.029 &  1.310 & 0.021 & 1.285  &	    0.011 &	1.295 &	0.019   \\
 Lasso    & 1.356 &	0.040 & 1.404 & 0.028 &	1.434 & 	0.034 &	1.372 &	0.043    \\
\hline
\end{tabular}
\end{center}
\end{table}

	%SPE	FSE	SPE	FSE	SPE	FSE	SPE	FSE
%RP-CD	1.273 	0.013 	1.247 	0.006 	1.241 	0.009 	1.327 	0.023
%RP-IR	1.292 	0.008 	1.242 	0.003 	1.225 	0.009 	1.319 	0.014
%								
%l_half	1.360 	0.003 	1.272 	0.002 	1.251 	0.003 	1.383 	0.051
%MCP	1.284 	0.012 	1.260 	0.014 	1.195 	0.010 	1.281 	0.016
%AdaLasso	1.310 	0.021 	1.285 	0.011 	1.295 	0.019 	1.360 	0.029
%Lasso	1.404 	0.028 	1.434 	0.034 	1.372 	0.043 	1.356 	0.040

Figure~\ref{fig:convergence} depicts the convergence procedure of  the coordinate descent iterations  with KEP,  $\ell_{1/2}$ and MCP, respectively. This figure shows that the coordinate descent algorithm is appropriate for  nonconvex penalty functions. Furthermore,
it is seen that the convergence speedups  with KEP and $\ell_{1/2}$ are competitive, but they faster than MCP. Specifically, to achieve convergence,  MCP usually needs to take about 50 steps  while both KEP and $\ell_{1/2}$ need to  take about 10 steps.
In summary, the KEP function with $q=1$ is a good choice in  nonconvex penalization and the KEP-CD method
is an efficient approach for solving the corresponding nonconvex  optimization problem.

\begin{figure}[!ht]
\centering
%% \begin{tabular}{ccc}
%%\hspace{-1.5cm}
\subfigure[$\ell_{1}$]{\includegraphics[width=75mm,height=60mm]{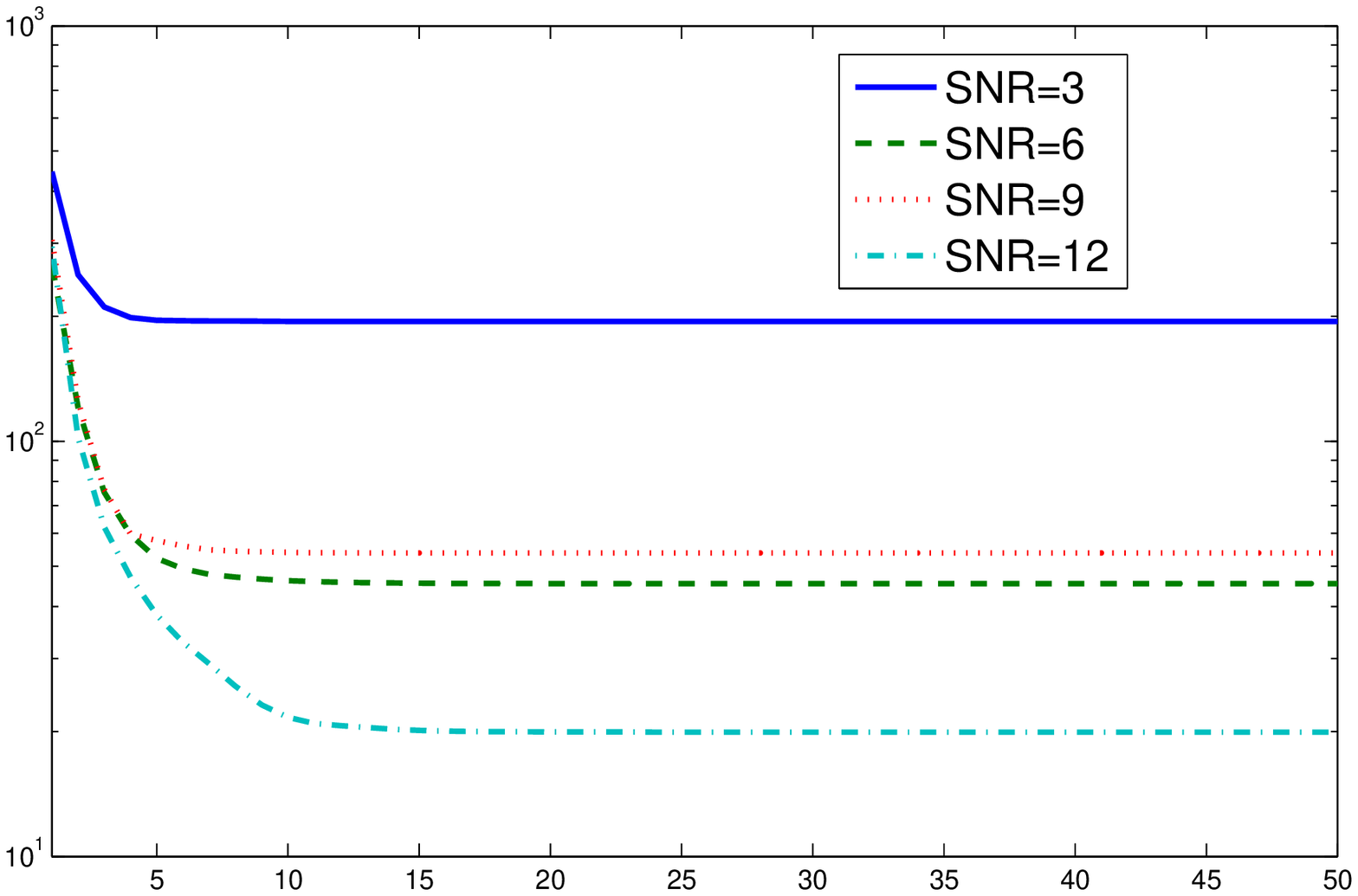}}
\subfigure[$\ell_{1/2}$]{\includegraphics[width=75mm,height=60mm]{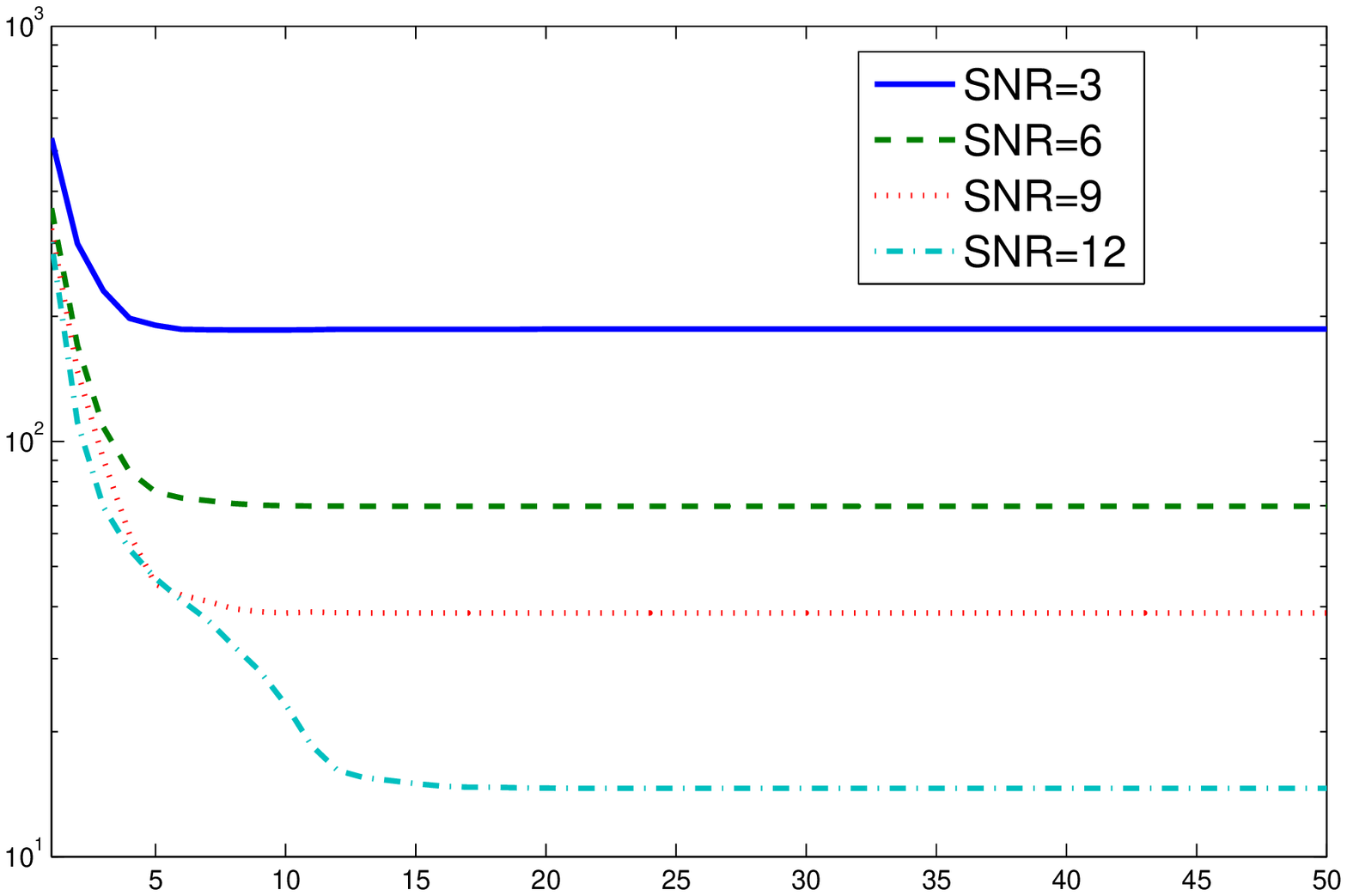}} \\
\subfigure[KEP]{\includegraphics[width=75mm,height=60mm]{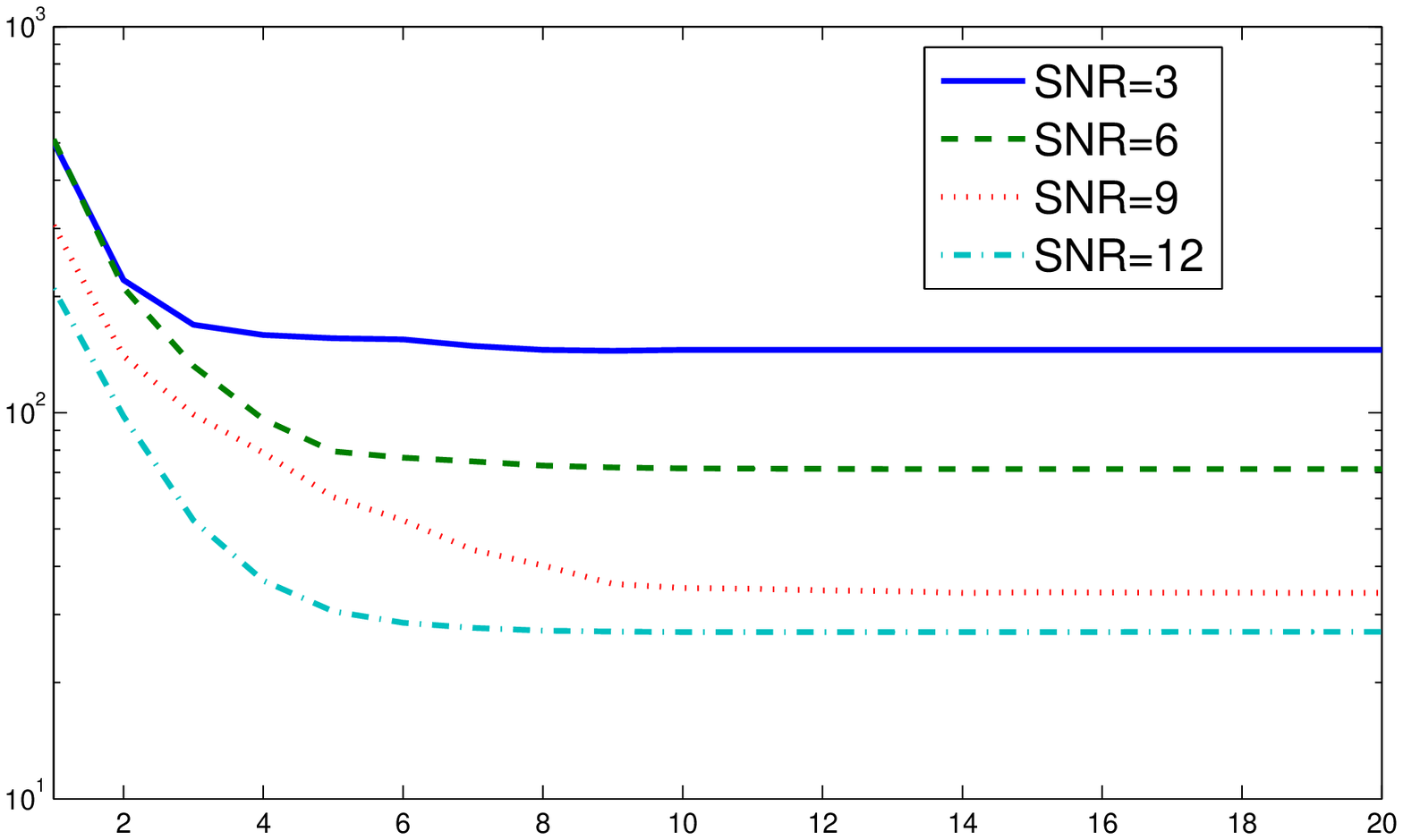}}
\subfigure[MCP]{\includegraphics[width=75mm,height=60mm]{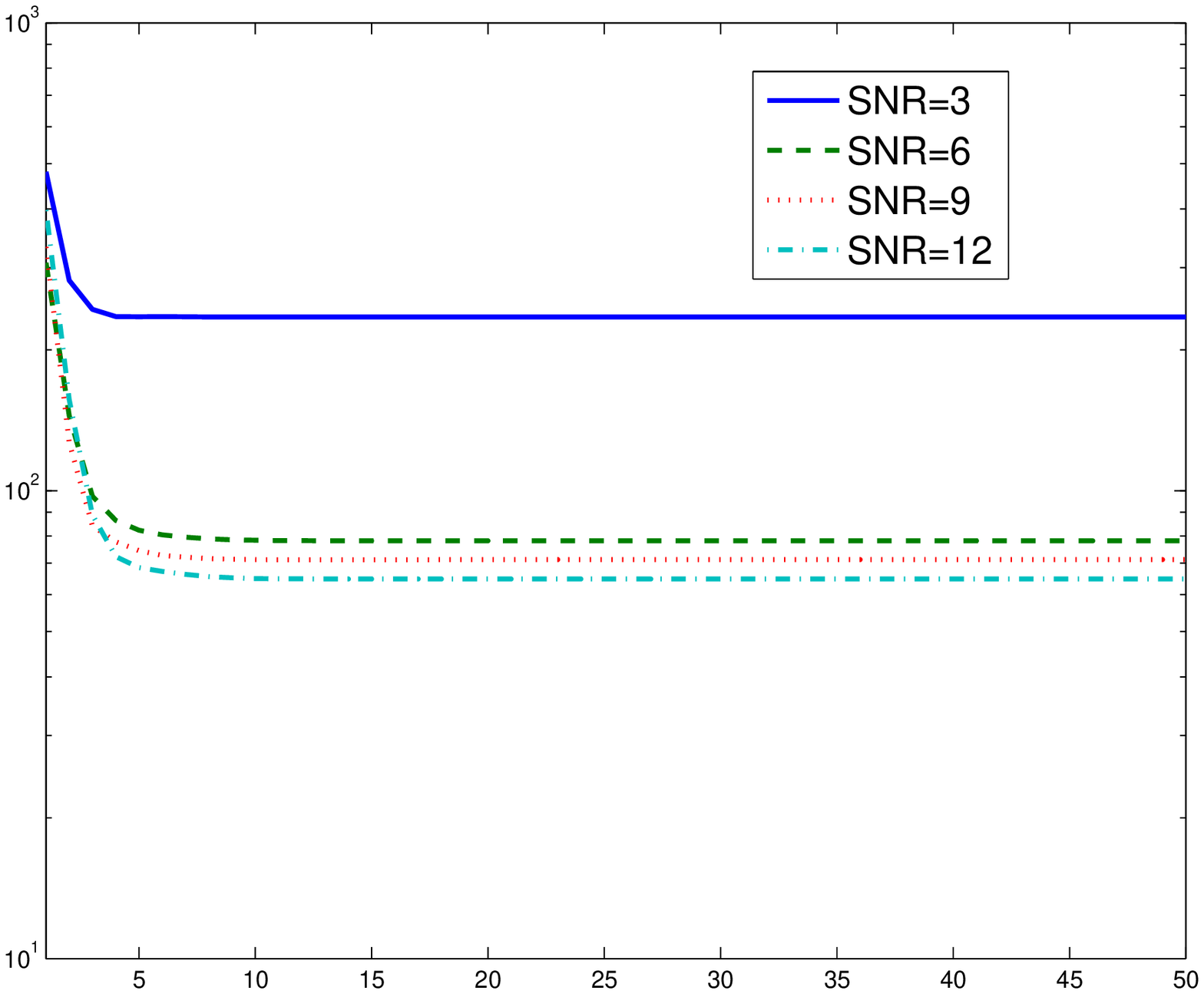}}\\
%\end{tabular}
\caption{Convergence procedures  for coordinate descent iterations with $\ell_1$, $\ell_{1/2}$,  KEP, and MCP-CD  over the datasets with $p=200$. Here the values of the y-axis are taken as log. }
\label{fig:convergence}
\end{figure}

\subsection{Linear Regression on Real Datasets} \label{sec:reg_real}

In this  experiment, we apply our methods to real regression problems on the cookie (Near-Infrared (NIR) Spectroscopy of Biscuit Doughs) dataset~\citep{Osborne84}. We  follow the setup of the original dataset: 39 instances for training and 31 instances for the test. Note that
the original dataset consists of 72 instances, but  two instances were suggested by~\cite{Brown:2001} to be excluded as outliers.
We train a model for each response among the four responses (``fat,'' ``sucrose,'' ``dry flour'' and ``water'') in the experiment.

We report the root mean square error (RMSE) on the test set and the model sparseness (proportion of zero coefficients in $\hat{\b}$) in Table \ref{tab:real_res4}. We can see that all the methods are competitive in  prediction accuracy. But in most cases the nonconvex
methods have  strong ability in feature selection. We can also see that performance of the method with KEP
is stable, while it is instable for the method with the $\ell_{1/2}$-norm penalty.
This agrees with the theoretical analysis in Section~\ref{sec:threshold}.
%choosing $\alpha = p$ is nearly the same as using cross-validation to choose $\alpha$,

%
%\begin{table}[!ht]  \setlength{\tabcolsep}{3.0pt}
%\caption{Root Mean Square Error (RMSE) and Model Sparseness ``SPR" (\%) on real datasets NIR where $n=70$ and $p=700$.}
%\label{tab:real_res3}
%\begin{center}
%
%\begin{tabular}{l|cc|cc|cc|cc} \hline
%    & \multicolumn{2}{c}{fat} & \multicolumn{2}{c}{sucrose}   & \multicolumn{2}{c}{flour } & \multicolumn{2}{c}{water} \\
%          & RMSE & ``SPR'' & RMSE & ``SPR'' & RMSE  & ``SPR" & RMSE & ``SPR''          \\
%\hline \\
% KEP-ALcv & 0.436 &   \textbf{99.29}       & 0.987  & \textbf{98.71}       & 1.083  & \textbf{99.71}  & 0.396 & 80.43 \\
% KEP-ALfp & 0.420 & \textbf{99.29}       & 0.987   & \textbf{98.71}       & 1.083  & \textbf{99.71}  & \textbf{0.360} & \textbf{98.43} \\
% KEP-ALap   & \textbf{0.296} & 94.71  & \textbf{0.951} & 95.14   & \textbf{0.722}  & 96.00  & 0.393 & 75.00  \\
%%   SRP-IR-DA & 0.359 &  1.078 & 0.683 & 0.388
%       \\
% Lasso    & 0.437 & 68.86  & 2.540  & 53.43  & 0.785  & 92.29           & 0.378 & 65.57  \\
%AdLasso   & 0.835 & 88.14  & 2.222  & 86.14  & 0.862  & 99.14           & 0.407  & 85.86 \\
%     MCP  & 0.943 & 94.14  & 2.069  & 95.43  & 0.839  & \textbf{99.71}  & 0.504 & 96.29 \\
%\hline
%\end{tabular}
%\end{center}
%\end{table}

\begin{table}[!ht]  \setlength{\tabcolsep}{3.0pt}
\caption{Root Mean Square Error (RMSE) and Model Sparseness ``SPR"  on real datasets NIR where $n=39$ and $p=700$.}
\label{tab:real_res4}
\begin{center}
\begin{tabular}{l|cc|cc|cc|cc} \hline
    & \multicolumn{2}{c}{fat} & \multicolumn{2}{c}{sucrose}   & \multicolumn{2}{c}{flour } & \multicolumn{2}{c}{water} \\
          & RMSE & ``SPR'' & RMSE & ``SPR'' & RMSE  & ``SPR" & RMSE & ``SPR''          \\
\hline \\
KEP-CD	&\textbf{0.4478} & 0.9914 & {1.1174} & 0.9871 & \textbf{0.6012} & 0.9914 & 0.4845 &	0.9929 \\
KEP-IR   & 0.5172	 & 0.9829    & \textbf{1.0677} &	0.9929 & 0.6808	& 0.9957 & \textbf{0.4458} &	0.9943 \\
\\
%MCP & 0.510 & 0.997 &	1.221 &	0.994	& 0.684 & 0.993 &	0.483	& 0.993
%\\
MCP	& 0.5170 &	0.9871 & 1.2163 & 0.9929 & 0.6250 & 0.9929 & 0.8780 & 0.9929 \\
$\ell_{1/2}$ & 0.6767 	& 0.9700 &	1.6353 & 0.9671 & 0.8211 &	0.9857 & 0.5642 & 0.9929
\\
AdaLasso & 0.5331  & 0.9843  &	1.1217 &  0.9743 & 	0.7304 &	0.9843 &	0.6881 & 0.9900
\\
Lasso	& 0.8177 &	0.9786 & 1.3601 & 0.9557 & 	0.8388 & 0.9729  &	0.5148  & 0.9857 \\
\hline
\end{tabular}
\end{center}
\end{table}

\section{Conclusion} \label{sec:conclusion}

In this paper we have studied  sparse penalized learning  problems.
We have focused on the iteratively reweighted $\ell_q$ method of \cite{Daubechies:2010} and  developed the kinetic energy plus  (KEP) penalty function.
In particular, we have illustrated that KEP can be defined as a concave conjugate of the nonnegative of a $\chi^2$-distance function.
We have thus rederived the  iteratively reweighted $\ell_q$ method of \cite{Daubechies:2010}.
%KEP can be used to induce prior distributions that we call \emph{relativity priors}.
%The prior  can
%expressed as a scale mixture of exponential power distributions with a generalized inverse Gaussian density~\citep{ZhangEPGIG:2012}. Furthermore,
%the relativity prior yields a strong posterior consistency~\citep{ArmaganBiometrika:2013}.

Under the setting of $q=1$, we have derived the thresholding operator for the KEP penalized univariate least-squares problem.
Accordingly, we have devised a coordinate descent algorithm.  We have validated that
this algorithm is effective and feasible in theoretically and empirically. Additionally, we have investigated
the relationship of  KEP with the $\ell_1$ and $\ell_{1/2}$ penalties. That is, the limiting cases
are the $\ell_1$ and $\ell_{1/2}$ penalties. Moreover, we have uncovered an interesting connection between
the KEP and MCP functions.  Specifically, the MCP function can be also defined as the concave conjugate of the $\chi^2$-distance function.
The difference between both them is due to asymmetricity of the $\chi^2$-distance function.

%\small{
%
%\bibliography{ncvs}
%\bibliographystyle{plain}
%}
%
%\end{document}

%%%%%%%%%%%%%%%%%%%%%%%%%%%%%%%%%%%%%%%%%%%%%%%%%%%%%%%%%%%%%%%%%%%%%%%%%%
%%%%%%%%%%%%%%%%%%%%%%%%%%%%%%%%%%%%%%%%%%%%%%%%%%%%%%%%%%%%%%%%%%%%%%%%%
\appendix

\section{The Proof of Theorem~\ref{thm:sparsty}}
\label{app:aa}

\begin{proof} The first-order derivative of (\ref{eqn:general}) w.r.t.\ $b$ is
\[
\sgn(b)\Big(|b| + {\eta} (2\alpha |b|+1)^{-\frac{1}{2}} \Big) - z.
\]
Let $g(|b|)= |b| + {\eta} (2 \alpha |b|+1)^{-\frac{1}{2}}$. It is clear that
$|z|< \min_{b\neq 0}\{g(|b| \}$, the resulting estimator is 0; namely, $\hat{b}=0$.
We now check the minimum value
of $g(s)=s + { \eta } (2\alpha s +1)^{-\frac{1}{2}}$ for $s\geq 0$.

Taking the first-order derivative of $g(s)$ w.r.t.\ $s$, we have
\[
g'(s) = 1 - { \eta \alpha} (2\alpha s +1)^{-\frac{3}{2}}.
\]
Thus, if $\eta \leq \frac{1}{\alpha}$, $g(s)$  attains its minimum value ${\eta}$ at $s^{*}=0$.
Otherwise,  $g(s)$ attains its minimum value when $s^{*}= \frac{1}{2 \alpha} \big[({\eta \alpha})^{2/3}-1\big]$; that is,
\[
g(s^{*})= \frac{3}{2 \alpha} \big({\alpha \eta}\big)^{\frac{2}{3}} {-} \frac{1}{2\alpha} \quad \Big(>\frac{1}{\alpha}\Big).
\]

First, we  consider the case that $\eta > \frac{1}{\alpha}$. In this case, the resulting estimator is 0 when $|z|\leq \frac{3}{2 \alpha} \big({\alpha \eta}\big)^{\frac{2}{3}} {-} \frac{1}{2\alpha}$. If $z> \frac{3}{2 \alpha} \big({\alpha \eta}\big)^{\frac{2}{3}} {-} \frac{1}{2\alpha}$, then the resulting estimator should be the positive root of the equation
$b + {\eta} (2\alpha b +1)^{-\frac{1}{2}}  - z = 0$ in $b$. Let $u=(2 \alpha b +1)^{\frac{1}{2}}$. We denote
\[
h(u) = u^3 - (2 \alpha z+1) u + {2 \eta \alpha}.
\]
Since $h((2\alpha z+1)^{1/2}) =  {2 \eta \alpha} >0$, $h(((2\alpha z+1)/3)^{1/2})
=-2 ((2\alpha z+1)/3)^{3/2} + {2 \eta \alpha} < - 2 {\alpha \eta} + {2 \eta \alpha}=0$,
 $h(0)= {2 \eta \alpha} >0$, and
\[
h\Big({-} \frac{2}{\sqrt{3}}(2 \alpha z{+}1)^{\frac{1}{2}} \Big) =
-\frac{2}{3 \sqrt{3}}(2\alpha z{+}1)^{\frac{3}{2}} + {2 \eta \alpha} <0,
\]
we have that cubic equation $h(u)=0$ has three reel roots. Moreover, the largest root (denoted ${u_0}$) is in $(((2\alpha z+1)/3)^{1/2}, (2\alpha z+1)^{1/2})$, which implies
that $1\leq {u_0}\leq(2\alpha z+1)^{1/2}$. As a result, the resulting estimator is $\hat{b}=\frac{{u_0}^2-1}{2\alpha} \leq z$.
Based on the trigonometric (and hyperbolic) method~\citep{NickallsCubic:2012}, ${u_0}$ is specified by
\[
{u_0} = 2 \sqrt{\frac{2 \alpha z{+}1}{3}} \cos\Big[\frac{1}{3}
\arccos\big( {-} {\alpha \eta} (\frac{3}{2\alpha z {+} 1})^{\frac{3}{2}}  \big) \Big].
\]
Similarly, if $z< - \frac{3}{2\alpha} ({\alpha \eta})^{\frac{2}{3}} {+} \frac{1}{2\alpha}$,
we can derive the analytic expression of the resulting estimator, which is given in (i).
Note that in this case of $\eta > \frac{1}{\alpha}$, the second largest root $u_1 \in (0, ((2\alpha z+1)/3)^{1/2})$.
This implies that $u_1> 1$ is possible.
If so, however,  the second root should corresponding to the maximum value of the original problem.  Therefore, in this case,
we still can prove the existence and uniqueness of the estimator $\hat{b}$.

Next, we consider the case that $\eta\leq \frac{1}{\alpha}$. In this case, the resulting estimator is 0 when $|z|\leq {\eta}$. If $z> {\eta}$, then the resulting estimator should be the positive root of the equation
$b + { \eta} (2 \alpha b +1)^{-\frac{1}{2}}  - z = 0$ in $b$. Accordingly, we study the roots of $h(u)=0$.
Note that
\[
\Delta= - 4(2 \alpha z{+}1)^3 + 27 t^2 \leq - 4(t{+}1)^3 + 27 t^2 = -(4 t^3 {-} 15t^2{+} 12t{+} 4) = -(t{-}2)^2(4t{+}1)\leq0
\]
where $t= {2\alpha \eta}\geq 0$. Thus, cubic equation $h(u)=0$ has three real roots. In fact,
we further have $h((2\alpha z+1)^{1/2}) =  {2 \eta \alpha} >0$, $h(1)= - 2\alpha z + { 2\alpha \eta }<0$,
$h(0)=  {2 \eta \alpha} >0$, and
\begin{align*}
h\big({-}(1+(2\alpha z{+}1)^{1/2})\big) & = -[1+(2\alpha z{+}1)^{1/2}]^3 + (2\alpha z {+}1)[1+(2\alpha z{+}1)^{1/2}] +  {2 \eta \alpha} \\
& \leq -[3(2\alpha z {+}1)^{1/2} +4 \alpha z +1] < 0.
\end{align*}
This implies that  $h(u)=0$
has one and only one root greater than 1, which belongs to $(1, (2\alpha z {+}1)^{1/2})$.
Consequently, the resulting estimator $0<\hat{b}<z$ when $z> {\eta}$.
Similarly, we can obtain that $z<\hat{b}<0$ when $z<- {\eta}$. Using the trigonometric  theory,
we can also obtain
an analytic formula for this root which is given the second part of the theorem.
%It is worth pointing out that in this case of  $\eta\leq \frac{4 \gamma}{\alpha^2}$, the estimator exists and is unique.
As stated in \cite{Fan01}, a sufficient and necessary condition for ``continuity" is
the minimum of
$|b|+  \Psi'(|b|)$ is attained at $0$. This implies that that the resulting estimator is continuous.
In fact, the continuity of the resulting estimator can also be obtained from Lemma~\ref{lem:33}-(ii) which is given below.
%\begin{align*}
%h(-(\alpha z{+}1)^2) & = -(\alpha z {+}1)^6 + (\alpha z {+}1)^3+  \frac{ \eta \alpha^2}{2 \gamma}
%\leq -(\alpha z {+}1)^6 + (\alpha z {+}1)^3+2 \\
%& =\big[1+ (\alpha z {+}1)^3 \big] \big[2- (\alpha z {+}1)^3 \big]<0.
%\end{align*}
%\begin{align*}
%h\Big({-}(\alpha z{+}1)^2 \frac{ \eta \alpha^2}{4 \gamma}\Big) &= \frac{ \eta \alpha^2}{4 \gamma}
%\Big[-(\alpha z {+}1)^6 (\frac{ \eta \alpha^2}{4 \gamma})^2 + (\alpha z+1)^3 + 2\Big]< \frac{ \eta \alpha^2}{4 \gamma}
%\Big[-(\alpha z {+}1)^6  + (\alpha z+1)^3 + 2\Big] \\
%& = \frac{ \eta \alpha^2}{4 \gamma} \big[2-(\alpha z {+}1)^3\big]  \big[1+ (\alpha z+1)^3\big]\leq 0,
%\end{align*}
%$h(1) =  1- (\alpha z+1) + \frac{ \eta \alpha^2}{2 \gamma} < 1 - 3
%\Big(\frac{\alpha^2 \eta}{4 \gamma}\Big)^{\frac{2}{3}}  + 2 \frac{ \eta \alpha^2}{4 \gamma}
%=\Big[\Big(\frac{\alpha^2 \eta}{4 \gamma}\Big)^{\frac{1}{3}} -1\Big]^2\Big[ 2\Big(\frac{\alpha^2 \eta}
%{4 \gamma}\Big)^{\frac{1}{3}} +1 \Big]$
\end{proof}

\begin{lemma} \label{lem:33} Given a $ t\geq 0$, we define
\[
\varphi(u)= 2 \sqrt{\frac{u{+}1}{3}} \cos\Big[\frac{1}{3}
\arccos\big( {-}{t} (\frac{u{+}1}{3})^{-\frac{3}{2}}  \big) \Big]
\]
for $u \geq 3 {t}^{\frac{2}{3}} -1$.
Then,
\begin{enumerate}
\item[\emph{(i)}]  $\varphi(u)$ and $\varphi^2(u)$ are strictly increasing on $[3 {t}^{\frac{2}{3}} -1, \infty)$.
\item[\emph{(ii)}] If $0 \leq t \leq 1$, then $\varphi(2 t)\equiv 1$.
\item[\emph{(iii)}] If $0 \leq t < 1$, then ${\varphi^2(u)}$ is  Lipschitz continuous on $[2t, \infty)$.
\end{enumerate}
\end{lemma}
\begin{proof} The first-order derivative of $\varphi(u)$ w.r.t.\ $u$ is
\[
\varphi'(u)= \frac{1}{3}(\frac{u{+}1}{3})^{-\frac{1}{2}} \cos\Big[\frac{1}{3}
\arccos\big(  {-} {t} (\frac{u{+}1}{3})^{{-}\frac{3}{2}}   \big) \Big] + \frac{\frac{t}{3} (\frac{u{+}1}{3})^{{-}{2}}} {\sqrt{1 {-} {t^2} (\frac{u{+}1}{3})^{{-}{3}}}} \sin\Big[\frac{1}{3}
\arccos\big( {-} {t} (\frac{u{+}1}{3})^{{-}\frac{3}{2}} \big) \Big],
\]
which is greater than 0.
%\begin{align*}
%g'(u) &= \frac{u^{-\frac{1}{2}}}{\sqrt{1 {-} t^{2} u^{-3}}} \left\{\sqrt{1 {-} t^{2} u^{-3}} \cos\Big[\frac{1}{3}
%\arccos\big( {-}{t} u^{-\frac{3}{2}}  \big) \Big] + {t u^{-\frac{3}{2}}}  \sin\Big[\frac{1}{3}
%\arccos\big( {-}{t} u^{-\frac{3}{2}}  \big) \Big]  \right\} \\
%&= \frac{u^{-\frac{1}{2}}}{\sqrt{1 {-} t^{2} u^{-3}}} \bigg\{\sin\arccos({t u^{-\frac{3}{2}}} )  \cos\Big[\frac{1}{3}
%\arccos\big( {-}{t} u^{-\frac{3}{2}}  \big) \Big] \\
%& \qquad \qquad \quad \qquad + \cos \arccos({t u^{-\frac{3}{2}}} ) \sin\Big[\frac{1}{3}
%\arccos\big( {-}{t} u^{-\frac{3}{2}}  \big) \Big]  \bigg\} \\
%&= \frac{u^{-\frac{1}{2}}}{\sqrt{1 {-} t^{2} u^{-3}}} \sin\Big[\arccos({t u^{-\frac{3}{2}}} ) + \frac{1}{3}
%\arccos\big( {-}{t} u^{-\frac{3}{2}}  \big) \Big] \\
%&= \frac{u^{-\frac{1}{2}}}{\sqrt{1 {-} t^{2} u^{-3}}} \sin\Big[\frac{\pi}{3} + \frac{2}{3} \arccos({t u^{-\frac{3}{2}}} ) \Big] >0.
%\end{align*}
Additionally, $\frac{d \varphi^2(u)}{d u} = 2 \varphi(u) \varphi'(u)>0$. Thus, $\varphi(u)$ and $\varphi^2(u)$ are strictly increasing.

For $0\leq t \leq 1$, it is directly verified that $1 \leq  \sqrt{\frac{3}{2t+1}} \leq {\sqrt{3}}$ and
\[
-1 \leq -t (\frac{3}{2t {+} 1})^{\frac{3}{2}} \leq 0.
\]
We thus can assume that $ \cos(\theta) =\frac{1}{2} \sqrt{\frac{3}{2t+1}}$ where $\theta \in [\pi/6, \pi/3]$. Since
\[\cos (3\theta) = 4 \cos^3(\theta) - 3 \cos(\theta) = -t \left( \frac{3}{2t {+}1} \right)^{\frac{3}{2}}
\]
and $3 \theta \in [\pi/2, \pi]$, we have
\[
\frac{1}{2} \sqrt{\frac{3}{2t{+}1}} = \cos(\theta)= \cos\Big[\frac{1}{3}
\arccos\big( {-}t (\frac{3}{2t {+} 1})^{\frac{3}{2}}  \big) \Big].
\]

Finally, we have
\[
\frac{d \varphi^2(u)}{d u} = \frac{2}{3}\Big\{1+  \cos\Big[\frac{2}{3}
\arccos\big(  {-}t (\frac{u{+}1}{3})^{{-}\frac{3}{2}}   \big) \Big] \Big\}+ \frac{\frac{2}{3} } {\sqrt{ \frac{1}{t^2}(\frac{u{+}1}{3})^{{3}} {-} 1} } \sin\Big[\frac{2}{3}
\arccos\big( {-}t (\frac{u{+}1}{3})^{{-}\frac{3}{2}} \big) \Big].
\]

Note that $\frac{1}{t^2}(\frac{u{+}1}{3})^{{3}} {-} 1 \geq \frac{1}{t^2}(\frac{2t{+}1}{3})^{{3}} {-} 1$ for $u \in [2t, \infty)$.
Moreover,  $\frac{1}{t^2}(\frac{2t{+}1}{3})^{{3}} {-} 1$ is strictly decreasing for $0\leq t <1$.
Thus, we have
\[
\Big|\frac{d \varphi^2(u)}{d u}\Big| \leq \frac{4}{3} + \frac{2}{3} \frac{1} {\sqrt{ \frac{1}{t^2}(\frac{2t{+}1}{3})^{{3}} {-} 1}}.
\]
\end{proof}

\section{The Proof of Theorem~\ref{thm:oracle1}}
\label{app:bb}

Let ${\b}_n^{(1)} =\b^{*} + \frac{\u}{\sqrt{n}}$ and
\[
\hat{\u} = \argmin_{\u} \; \bigg\{G_n(\u) := \Big\|\y - \X \big(\b^{*} +\frac{\u}{\sqrt{n}} \big) \Big\|^2  +  \sum_{j=1}^p \omega_j^{(0)} \Big|b^{*}_j {+} \frac{u_j}{\sqrt{n}}\Big| \bigg\},
\]
where $\omega_j^{(0)} =  \frac{\eta_n}{ \sqrt{1{+}2\alpha_n |b_j^{(0)}|}}$.
%Then ${\u} = \sqrt{n}({\b}_n^{(1)} {-} \b^{*})$.
Consider that
\[
 G_n(\u) - G_n(0)
= \u^T(\frac{1}{n}\X^T \X) \u {-} 2 \frac{\epsi^T \X}{\sqrt{n}} \u {+}  \sum_{j=1}^p
\omega_j^{(0)} \Big \{ \big|b^{*}_j {+} \frac{u_j}{\sqrt{n}} \big| {-} |b^{*}_j| \Big\}.
\]
We know that $\X^T\X/n \rightarrow \C$ and $\frac{ \X^T \epsi}{\sqrt{n}} \rightarrow_{d} N(\0, \sigma^2 \C)$. We thus only consider
the third term of the right-hand side of the above equation.
If $b^{*}_j=0$, then $\sqrt{n} \big(|b^{*}_j+\frac{u_j}{\sqrt{n}}|-|b^{*}_j|\big)=|u_j|$. And since $\sqrt{n} b_j^{(0)}=O_p(1)$, we have $\alpha_n |b^{(0)}_j|
= (\alpha_n/\sqrt{n}) \sqrt{n} \big|b^{(0)}_j\big| =O_p(1)$. Hence,
\[
\frac{ \omega^{(0)}_j}{\sqrt{n}}   \rightarrow_p   \frac{\eta_n}{n^{1/2}}    \rightarrow \infty
\]
due to $\alpha_n/\sqrt{n} \rightarrow c_1 \in (0, \infty)$.
%(Note that implies $\gamma_n \rightarrow +\infty$).
If $b^{*}_j\neq0$, then %$\omega^{(0)}_j\rightarrow_p \frac{c_2} { \sqrt{|b^{(0)}_j|}} >0$
\[
\frac{ \omega^{(0)}_j}{\sqrt{n}} = \frac{\eta_n} {n^{3/4}} \frac{1} {  \sqrt{\frac{1}{\sqrt{n}}{+} \frac{2 \alpha_n}{\sqrt{n}} |b_j^{(0)}|}}
\rightarrow_p \frac{\eta_n} {n^{3/4}} \rightarrow 0
\]
and  $\sqrt{n}(|b^{*}_j+\frac{u_j}{\sqrt{n}}|-|b^{*}_j|)\rightarrow u_j \sgn(b^{*}_j)$.
Thus, $ \frac{\omega^{(0)}_j}{\sqrt{n}}\sqrt{n}(|b^{*}_j {+} \frac{u_j}{\sqrt{n}}| - |b^{*}_j|)\rightarrow_p 0$.
The remaining parts of the proof  can be immediately obtained via some slight modifications to that in \cite{ZouJasa:2006} or
\cite{ZouLi:2008}.  We here omit them.

\section{The Proof of Theorem~\ref{thm:oracle2}}
\label{app:cc}

Let $\tilde{\b}_n =\b^{*} + \frac{\u}{\sqrt{n}}$ and
\[
\hat{\u} = \argmin_{\u} \; \bigg\{G_n(\u) := \Big\|\y - \X \Big(\b^{*} {+} \frac{\u}{\sqrt{n}} \Big) \Big\|^2
    + \frac{\eta_n}{\alpha_n} \sum_{j=1}^p \Big[\sqrt{1 {+} 2 \alpha_n \Big|b^{*}_j {+} \frac{u_j}{\sqrt{n}} \Big|} -1\Big] \bigg\}.
\]
%Then ${\u} = \sqrt{n}(\tilde{\b}_n - \b^{*})$.
Consider that
\[
 G_n(\u) - G_n(0)
 = \u^T(\X^T \X/n) \u - 2 \frac{\epsi^T \X}{\sqrt{n}} \u +
  \frac{\eta_n}{\alpha_n} \sum_{j=1}^p
\Big[{\sqrt{1 {+} 2\alpha_n \Big|b^{*}_j {+} \frac{u_j}{\sqrt{n}} \Big|} {-} \sqrt{1 {+} 2\alpha_n|b^{*}_j|} }\Big].
\]
Clearly, $\X^T \X/n \rightarrow \C$ and $\frac{\X^T\epsi}{\sqrt{n}} \overset{d}{\rightarrow} \z  \overset{d}{=} N(\0, \sigma^2 \C)$.
We now discuss the limiting behavior of the third term of the right-hand side. We partition $\z$ into $\z^T=(\z_1^T, \z_2^T)$ where  $\z_1=\{z_j: j \in {\cal A}\}$ and $\z_2=\{z_j: j \notin {\cal A}\}$.

First, consider the case that $b^{*}_j=0$. In this case,
we have
\[
 \frac{\eta_n}{\alpha_n}\Big[{\sqrt{1 {+} 2 |u_j| \frac{ \alpha_n}{\sqrt{n}}} - 1} \Big]
= \frac{\eta_n}{\sqrt{n}}
\frac{2 |u_j|} {\sqrt{1 {+} |u_j| \frac{2\alpha_n}{\sqrt{n}}} + 1} \rightarrow +\infty.
%& = u_j \frac{\eta_n} {n^{1/4} }
%\frac { (\frac{1}{\gamma_n} + \sqrt{1 + \frac{1}{\gamma_n}} ) \sqrt{\frac{\gamma_n}{\sqrt{n}} } }  {\sqrt{1 {+} |u_j|
%\frac{\gamma_n}{\sqrt{n}} } + 1}
\]
Second, we assume that $b^{*}_j\neq 0$. Subsequently,
\begin{align*}
&\frac{\eta_n}{\alpha_n} \Big[{\sqrt{1 {+} 2\alpha_n|b^{*}_j {+} \frac{u_j}{\sqrt{n}}|} - \sqrt{1 {+} 2\alpha_n|b^{*}_j|} } \Big]  \\
 & \rightarrow   \frac{\eta_n}{\alpha_n} \Big[{\sqrt{1 {+} 2\alpha_n (b^{*}_j {+} \frac{u_j}{\sqrt{n}}) \sgn(b^{*}_j) } {-} \sqrt{1 {+} 2\alpha_n b^{*}_j \sgn(b^{*}_j) } } \Big]  \\
%& = \frac{\eta_n}{\sqrt{n}} \frac{ \sqrt{1+\gamma_n} + 1 } { \sqrt{1 {+} \gamma_n b^{*}_j \sgn(b^{*}_j) } } \frac{u_j \sgn(b^{*}_j) }
%{ 1+ \sqrt{1 {+} \frac{u_j \sgn(b^{*}_j) /\sqrt{n}} {1{+} \gamma_n b^{*}_j \sgn(b^{*}_j)} } }  \\
& = \frac{\eta_n}{{n^{3/4}}}  \frac{2 u_j \sgn(b^{*}_j) }
{\sqrt{ [{1}{ {+} 2 \alpha_n b^{*}_j \sgn(b^{*}_j)] /\sqrt{n}}} +   \sqrt{[1 {+} 2\alpha_n (b^{*}_j {+} \frac{u_j}{\sqrt{n}}) \sgn(b^{*}_j) ]/\sqrt{n} } }  \\
& \rightarrow 0.
\end{align*}

By Slutsky's theorem, we have
\[
 G_n(\u) - G_n(0)  \overset{d}{\rightarrow} \left\{\begin{array}{ll} \u_1^T \C_{11} \u_1 - 2 \u_1^T  \z_1 & \mbox{ if } u_j=0 \; \forall j \notin {\cal A},
 \\ \infty & \mbox{ otherwise}. \end{array} \right.
\]
This implies that $G_n(\u) - G_n(0)$ converges in distribution to a convex function, whose unique minimum is
$(\C_{11}^{-1} \z_1, \0)^T$. It then follows from  epiconvergence \citep{GeyerANS:1994,KnightFu:2000} that
\begin{equation} \label{eqn:11}
\hat{\u}_1 \overset{d}{\rightarrow} \C_{11}^{-1} \z_1 \; \mbox{ and } \;  \hat{\u}_2 \overset{d}{\rightarrow} \0.
\end{equation}
This proves asymptotic normality due to $\z_1 \overset{d}{=} N(\0, \sigma^2 \C_{11})$.
%Moreover, we have  $\sqrt{n} |\tilde{b}_{nj}| \overset{p}{\rightarrow}  0$.

Recall that $\tilde{b}_{n j} \overset{p}{\rightarrow} b^{*}_j$ for any $j \in \AM$, which implies that $\Pr(j \in \AM_n) \to 1$.
Thus, for consistency in Part (1), it suffices to obtain  $\Pr(l \in \AM_n) \to 0$ for any $l \notin \AM$.
For such an event ``$l\in \AM_n$," it follows from the KKT optimality conditions
that $2 \x_{\cdot l}^T(\y - \X \tilde{\b}_{n})= \frac{\eta_n } {\sqrt{1+ 2 \alpha_n |\tilde{b}_{nl}|}}$ where $\x_{\cdot l}$ is the $l$th column of $\X$.
Note that
\[
\frac{2 \x_{\cdot l}^T (\y -\X \tilde{\b}_{n})}{\sqrt{n}} = 2 \frac{\x_{\cdot l}^T \X \sqrt{n}(\b^{*}-\tilde{\b}_n)}{n} + \frac{2 \x_{\cdot l}^T \epsi}{\sqrt{n}}
\]
and $ \liml_{n\to \infty} \frac{\eta_n } {\sqrt{n} \sqrt{1+ 2 \alpha_n |\tilde{b}_{nl}|}}
 \to \infty$
due to $\sqrt{n} |\tilde{b}_{nj}| \overset{p}{\rightarrow}  0$ by (\ref{eqn:11}) and Slutsky's theorem. Accordingly, we have
\[
\Pr(l \in \AM_n) \leq \Pr\Big[ 2 \x_{\cdot l}^T(\y - \X \tilde{\b}_{n})=\frac{\eta_n } { \sqrt{1+ 2 \alpha_n |\tilde{b}_{nl}|}} \Big]
\to 0.
\]

\section{The Proof of Theorem~\ref{thm:asymptotic}}
\label{app:dd}

As for the proof of Theorem~\ref{thm:asymptotic}, we consider the case that $\liml_{n\to \infty} \alpha_n =0$. In this case, we have
\[
\lim_{n \to \infty} \frac{\sqrt{1+ (2\alpha_n/\sqrt{n})} -1} {\alpha_n/\sqrt{n}} = 1 \; \mbox{ and } \;
\lim_{n \to \infty} \frac{\sqrt{1+ 2\alpha_n} -1}{\alpha_n} = 1.
\]
Assume that $\liml_{n\to \infty} \eta_n/\sqrt{n}= 2 c_3 \in [0, \infty]$. Then
\[
\frac{\eta_n} {\alpha_n} \Big[{\sqrt{1+ 2 |u_j| \frac{\alpha_n}{\sqrt{n}}} -1} \Big]=  |u_j|  \frac{\eta_n}{\sqrt{n}}
 \frac{\sqrt{1+ 2 |u_j| \frac{\alpha_n}{\sqrt{n}}} -1} { |u_j|{\alpha_n}/{\sqrt{n}}} \to 2 c_3 |u_j|
\]
when $u_j\neq 0$. If $b^{*}_j\neq 0$, then
\begin{align*}
& \eta_n \frac{\sqrt{1 + 2(\alpha_n|b^{*}_j {+} \frac{u_j}{\sqrt{n}}|)} - \sqrt{1 + 2 \alpha_n|b^{*}_j|} } {\alpha_n} \nonumber  \\
 & =   \eta_n \frac{\sqrt{1+2\alpha_n (b^{*}_j {+} \frac{u_j}{\sqrt{n}})
 \sgn(b^{*}_j) } {-} \sqrt{1+2 \alpha_n b^{*}_j \sgn(b^{*}_j)} } {\alpha_n}  \nonumber \\
& =   \frac{\eta_n }{\sqrt{n}}  \frac{2 {u_j} \sgn(b^{*}_j)} {\sqrt{1+2\alpha_n (b^{*}_j {+} \frac{u_j}{\sqrt{n}})
 \sgn(b^{*}_j) } {+} \sqrt{1+2 \alpha_n b^{*}_j \sgn(b^{*}_j)}}  \\
& \to 2 c_3 u_j \sgn(b^{*}_j).
\end{align*}
We now first consider the case that $c_3 = 0$. In this case,
we have
\[
 G_n(\u) - G_n(\0)  \overset{d}{\longrightarrow} \u^T \C \u - 2 \u^T  \z,
\]
which is convex w.r.t.\ $\u$. Then the minimizer of $\u^T \C \u {-} 2 \u^T  \z$ is $\u^{*}$ if and only if  $\C \u^{*} - \z=\0$. Since
$\hat{\u} \overset{d}{\rightarrow} \u^{*}$ (by epiconvergence), we obtain $ \sqrt{n}(\tilde{\b}_n - \b^{*})=\hat{\u}  \overset{d}{\rightarrow}
N(\0, \sigma^2 \C^{-1})$.

We then consider the case that $c_3 \in (0, \infty)$. Right now we have
\[
 G_n(\u) - G_n(\0)  \overset{d}{\longrightarrow} \u^T \C \u - 2 \u^T  \z+ 2c_3 \sum_{j \in \AM} u_j \sgn(b^{*}_j) +
2 c_3 \sum_{j \notin \AM}  |u_j|  \triangleq H_2(\u).
\]
$H_2(\u)$ is convex in $\u$. Let the minimizer of $H_2(\u)$ be $\u^{*}$. Then
\[
\C \u^{*} - \z + c_3 \s =0
\]
where $\s^T = (\sgn(\b^{*}_1)^T, \v^T)$ and $\v \in \RB^{p_2}$ with $\max_{j} |v_j| \leq 1$. Thus,
we have $\u^{*} \overset{d}{\rightarrow} N({\bf t}, \sigma^2 \Tha)$ where ${\bf t}=(t_1, \ldots, t_p)^T=-c_3 \C^{-1} \s$
and $\Tha=[\theta_{ij}]= \C^{-1}$.
For any $\epsilon>0$, when $n$ is significantly large and using Chebyshev's inequality,  we have that
\begin{align*}
\Pr\Big[|u_j^{*}|/\sqrt{n} \geq \epsilon \Big]
&= \Pr\Big[ |u_j^{*}| \geq \sqrt{n} \epsilon \Big] \\
& \leq
\Pr\Big[|u_j^{*} - t_j| \geq \sqrt{n} \epsilon - |t_j| \Big]
 \leq \frac{\sigma^2 \theta_{jj}}{(\sqrt{n} \epsilon - |t_j| )^2} \to 0
\end{align*}
for $j=1, \ldots, p$. Consequently,  $|u_j^{*}|/\sqrt{n} \overset{p}{\rightarrow}0$;
that is, $\tilde{\b}_n  \overset{p}{\rightarrow} \b^{*}$.

\bibliography{ncvs3}
%\bibliographystyle{plain}

%}
%\end{small}

\end{document}